\documentclass[conference]{IEEEtran}
\IEEEoverridecommandlockouts

\usepackage{amsmath,amsfonts,bm}

\def\eqref#1{equation~\ref{#1}}

\def\1{\bm{1}}

\DeclareMathAlphabet{\mathsfit}{\encodingdefault}{\sfdefault}{m}{sl}
\SetMathAlphabet{\mathsfit}{bold}{\encodingdefault}{\sfdefault}{bx}{n}

\usepackage{hyperref}
\usepackage{url}
\usepackage{amsmath}
\usepackage{amssymb}
\usepackage{graphicx}
\usepackage[most]{tcolorbox}
\usepackage{amsthm}
\usepackage{booktabs}
\usepackage{threeparttable}

\newtheorem{theorem}{Theorem}
\newtheorem{proposition}{Proposition}

\newtcolorbox{insight}{
  colback=white,
  colframe=black!40,
  fonttitle=\bfseries,
  title=Insight,
  boxrule=0.4pt,
  arc=0pt,
  left=6pt, right=6pt, top=6pt, bottom=6pt
}

\title{The Layer Mystery of VLA: An Information-Theoretical Analysis of VLA Latent Interface}

\author{\IEEEauthorblockN{Yuxiang Liu\IEEEauthorrefmark{1},
Lizhi Yang\IEEEauthorrefmark{2},
Fengze Xie\IEEEauthorrefmark{2},
Aaron Ames\IEEEauthorrefmark{2}, and
Yisong Yue\IEEEauthorrefmark{2}}
\thanks{\IEEEauthorrefmark{1}University of California, Berkeley}
\thanks{\IEEEauthorrefmark{2}California Institute of Technology}
}

\begin{document}

\maketitle

\begin{abstract}
Vision-language-action (VLA) policies connect a pretrained vision-language backbone to an action head through a latent interface, but which backbone layers this interface should expose remains unclear. We study single-layer selection and multi-layer fusion for frozen backbones across three pretrained models and two manipulation benchmarks, LIBERO and CALVIN, with three policy-training seeds per configuration. Across three fusion mechanisms and three layer-subset strategies, 47 of 54 configurations underperform the best observed single-layer policy. Our stastical analysis further confirms that fusion's advantage is very limited. However, the best layer varies substantially across backbones and benchmarks, making layer selection consequential and exhaustive policy sweeps expensive. We further derive a reweighting equivalence between the proposed information-bottleneck objectives for action-conditioned InfoNCE and action prediction, motivating InfoNCE as a proxy for layer quality. Empirically, InfoNCE provides the most consistent positive association with policy success among four evaluated proxies. Selecting the layer with the highest InfoNCE score requires 9–33 times less GPU compute than exhaustive policy sweeps and reduces mean selection regret from 17.89 percentage points for deepest-layer selection to 3.71 points across six settings. Its mean regret is close to the 3.28–3.50 points achieved by fixed-layer heuristics optimized retrospectively using all six oracle sweeps, without requiring closed-loop evaluations during selection.
\end{abstract}

\section{Introduction}
\label{sec:intro}
Vision-Language-Action (VLA) models attach an action head to a pre-trained
vision-language model (VLM) backbone. The interface between the two is a set of feature
tokens read out of the backbone, and a designer must decide which of the
backbone's layers those tokens come from. The decision is consequential: it
fixes what information the policy can condition on, and it determines how much
of the backbone must be evaluated at inference time. For the rest of this paper, we will refer to this interface as \textbf{latent interface}.

The field does not agree on how to make it, and two paradigms prevail.
First, the \textbf{single-layer paradigm} uses the outputs of a single layer of the VLM backbone as the latent interface. Its members disagree about which layer. GR00T~N1 reads from the 12th layer of its 24-layer language tower, reporting that middle-layer embeddings gave ``both faster inference speed and
higher downstream policy success rate'' than final-layer
embeddings~\cite{bjorck2025gr00t}. RS-CL taps layer 18 of
28~\cite{kim2026contrastiverepresentationregularizationvisionlanguageaction}. SmolVLA discards the top half of its backbone outright,
retaining 16 of 32 layers~\cite{shukor2025smolvla}. Against these, OpenVLA and
its regression-head variant read final-layer hidden states, though as a
consequence of being single-tower models rather than as a considered
choice~\cite{kim2024openvla,kim2025openvlaoft}. ABot-M0, whose backbone has
undergone large-scale robot pre-training, argues for the final layer explicitly,
reporting the reverse of GR00T~N1: deep features beat intermediate
ones~\cite{yang2026abotm0vlafoundationmodel}.
Second, the \textbf{layer-fusion paradigm} uses outputs from multiple VLM layers
and fuses them to construct the latent interface. $\pi_0$'s action expert
cross-attends to every backbone layer by construction, never selecting
one~\cite{black2024pi0}. VLA-Adapter fuses all layers through a gated Bridge
Attention module and reports that the fusion outperforms every single layer it
tested---while also finding, among single layers, that a middle layer beats a
deep one for raw backbone features~\cite{wang2025vlaadapter}. Here too the
robot-pre-trained case dissents: ABot-M0 finds that aggregating its last sixteen
layers is \emph{worse} than the final layer alone~\cite{yang2026abotm0vlafoundationmodel}.

These positions are not obviously reconcilable, and the evidence offered for
them is thinner than their influence suggests. GR00T~N1's claim did not report any ablation table; VLA-Adapter's sweep, the most complete in the
literature, separates its best single layer from its runner-up by 1.4 points, lower than the cross-seed performance standard deviation.
ABot-M0 compares three coarse depth buckets. RS-CL compares three layers.
More importantly, across every one of these studies, each configuration is trained only once. To the best of our knowledge, no paper
reports multiple training seeds or error bars. Differences of one to three points on a benchmark where
strong models score in the high eighties are being used to justify architectural
defaults, without any estimate of the noise those differences sit in.

The interpretability literature provide relevant insights, but does not test them on VLA policies. Intermediate
transformer layers frequently yield better representations than final ones for
downstream tasks~\cite{skean2025layer}; deep networks develop a late
compressive ``tunnel'' that degrades transfer~\cite{masarczyk2023tunnel};
intermediate features generalize better under distribution
shift~\cite{uselis2025ilc}. This work
evaluates linear probes on classification and embedding tasks. Whether the
information it finds decodable at mid-depth is information a robot policy
actually uses is a separate question, and an untested one.

In conclusion, three questions remain unclear for the latent interface design in VLA community:

\begin{itemize}
\item Should the latent interface be designed following single-layer paradigm or layer-fusion paradigm?
\item For single-layer paradigm, how to choose the layer?
\item For layer-fusion paradigm, how to design the fusion strategy?
\end{itemize}

In this paper, our aim is to rigorously verify the first two questions and motivate potential future research paths to the third question. Our contribution of can be broken down as follow:

\begin{itemize}
\item Our empirical experiments systematically compare existing layer-fusion methods against single-layer latent interface and show the limitations of existing fusion methods.
\item Our exhaustive layer-sweep experiments validate the significance of a computation-efficient layer-selection scheme and propose infoNCE as a strong candidates empirically.
\item Our analysis based on information-bottleneck constructions provides a theoretical lens to interpret how optimal latent representations of action prediction relate to infoNCE.
\end{itemize}

\section{Problem Formulation and Preliminary}
\label{sec:formulation}
\subsection{Problem Formulation}
In this section, we define necessary notations and introduce theoretical foundations. We denote the expert action as a random variable $A \sim P_A$ where $P_A$ is the distribution of expert action. Denote the model input(RGB observations + text command) as random variable $X \sim P_X$, where $P_X$ is the input distribution. We denote the joint distribution of $X$ and $A$ as $P_{X,A}$.

We define our full VLA network as a composite function $f = g \; \circ \phi$ where $g$ is the action head and $\phi$ is any functional composition of outputs from any layers of the VLM backbone. The outputs of $\phi$ is exactly the latent interface per our definition. Specifically, let $Z = \phi(X)$ be the random variable for latent vector whose distribution is obtained by pushing $X$ forward through $\phi$. The joint distribution of $Z$ and $A$ is denoted as $P_{Z,A}$

Given $X, \; A, \; \phi$, the goal of the action head is to minimize prediction error:

\begin{equation}
\label{eq:action_pred_obj}
\mathcal{L}_{\text{Action}}(g)=\mathbb{E}_{(x, a)\sim P_{X,A}}[\| \; g(\phi(x))-a \; \|]
\end{equation}

Let our VLM backbone of $D$ layers be a composite function $\psi_D \;\circ \; \psi_{D -1} \; \circ\; \psi_{D- 2} \; \circ \; \cdots \; \circ \psi_1$ ~\cite{kawaguchi2023doesinformationbottleneckhelp, goldfeld2019estimatinginformationflowdeep}. With a little bit of abuse of notations, denote the outputs of $l \text{th}$ layer of the VLM backbone as random variable $Z_l = \psi_l \;\circ \; \psi_{l -1} \; \cdots \; \circ \psi_1 (X)$, which is a push-forward distribution obtained by $X$ through the first $l$ layers of the VLM backbone. Similarly, we define $\phi_l = \psi_l \;\circ \; \psi_{l -1} \; \cdots \; \circ \psi_1$. We also make the assumption that each layer of the VLM backbone $\psi_i, \; i\in \{1, \cdots D\}$ is deterministic function without stochasticity which is reasonable for the majority of attention-based architectures.

One critical question to be addressed before any argument is: How to define a desirable latent interface $Z$ induced by $\phi$? We directly borrow the concept of information bottleneck(IB) objective from the community of statistical learning ~\cite{oord2019representationlearningcontrastivepredictive, alemi2019deepvariationalinformationbottleneck, poole2019variationalboundsmutualinformation,kawaguchi2023doesinformationbottleneckhelp, skean2023dimemaximizingmutualinformation}:

\begin{proposition}[Latent Objective for Action Prediction]
\label{prop1:info_action_objective}
Given $X$ and $A$, a desirable latent interface $Z =\phi(X)$ for action prediction problem (Eq. ~\ref{eq:action_pred_obj}) should solve the following optimization problem:

\begin{equation}
\label{eq:info_objective}
\max_{\phi} \;  I(Z, A) - \alpha \; I(Z, X \; | \; A)
\end{equation}
where $I$ denotes mutual information and $\alpha$ is a constant multiplier.
(Adopted from \cite{alemi2019deepvariationalinformationbottleneck} and \cite{kawaguchi2023doesinformationbottleneckhelp})

\end{proposition}
To see how we modify the original proposition in ~\cite{alemi2019deepvariationalinformationbottleneck} by taking the insights of ~\cite{kawaguchi2023doesinformationbottleneckhelp}, check Appendix ~\ref{Appendix:defend_prop1}. Two core terms are involved: $I(Z, A)$ and $I(Z, X | A)$. \begin{itemize}
    \item $I(Z, A)$ quantifies how much information about the expert actions is encoded in latent interface $Z$. For example, information to be accounted by this term may include the positions of the target object to be picked up in a manipulation task.

    \item $I(Z, X | A)$ measures how much information encoded in $Z$ is action-irrelevant noise. Per ~\cite{kawaguchi2023doesinformationbottleneckhelp}, a latent interface with high $I(Z, X |A)$ will pose higher difficulty to the training of action head $g$ in terms of generalization error.
\end{itemize}

\subsection{The InfoNCE Objective and Its Information Bottleneck}

This section reviews infoNCE \cite{oord2019representationlearningcontrastivepredictive}. Fix a batch size $N$. A positive index Y is sampled uniform on $\{1, \cdots, N \}$. Then, $A'$ and $X' = (X_1, \ldots, X_N)$ are drawn as follow:
\begin{itemize}

\item We draw $A'$ and Y-indexed $X_Y$ sample input jointly. $(X_Y, A') \sim P_{X,A}$,

\item The remaining $N-1$ samples in $X'$ other than $X_Y$ are drawn independently from the marginal $P_X$ and are thus independent of $A'$. $X_i \sim P_X \; \forall\; i \neq Y, X_i \in X'$

\end{itemize}
Pushing each candidate through the latent interface gives $Z' = (\phi(X_1), \ldots, \phi(X_N))$. With a critic $h$ scoring latent--action pairs, the InfoNCE loss is a softmax over the $N$ candidates:
\begin{equation}
\label{eq:infonce_loss}
\mathcal{L}_{\text{InfoNCE}}(\phi, h) = -\, \mathbb{E}_{}\left[ \log \frac{\exp\!\big(h(\phi(X_Y), A')\big)}{\sum_{i=1}^{N} \exp\!\big(h(\phi(X_i), A')\big)} \right].
\end{equation}

InfoNCE is an $N$-way classification problem whose target label is $Y \in \{1, \ldots, N\}$. As in the action-prediction problem, a desirable latent interface $Z'$ should retain the information needed for this classification while discarding the rest. Here the classification signal is what $Z'$ and $A'$ jointly reveal about $Y$, and the irrelevant noise is the input information $Z'$ carries beyond what is needed to identify the $Y$ label. This mirrors Proposition~\ref{prop1:info_action_objective} and motivates the following objective.

\begin{proposition}[Latent Objective for InfoNCE]
\label{prop2:infonce_objective}
With $X', A', Y$ defined as above and $Z' = (\phi(X_1), \ldots, \phi(X_N))$, a desirable latent interface for the InfoNCE problem (Eq.~\ref{eq:infonce_loss}) solves
\begin{equation}
\label{eq:infonce_ib_objective}
\max_{\phi} \; I(Z'; A', Y) \;-\; \beta \, I(Z'; X' \mid A', Y).
\end{equation}
\end{proposition}

The two terms mirror Proposition~\ref{prop1:info_action_objective}. See ~\ref{app:alt_prop2_construct} for alternative definitions of Prop. ~\ref{prop2:infonce_objective}.
\begin{itemize}

\item $I(Z'; A', Y)$ is the classification-relevant signal, mirroring $I(Z, A)$ in Prop. ~\ref{prop1:info_action_objective}.

\item $I(Z';X'\mid A', Y)$ stands for the information encoded by $\phi$ that is irrelevant to classification. This term mirrors $I(Z,X|A)$.
\end{itemize}

\section{Single-layer or Layer-Fusion}
\label{Sec:paradigm_choice}
This section aims to address the problem:
should the latent interface be designed following single-layer paradigm or layer-fusion paradigm?

\subsection{Experiment}
\label{Sec:paradigm_choice_exp}

\begin{insight}
47 out of 54 possible combinations of fusion methods and layer configurations present empirical performance lower than the single optimal layer as a single-layer interface.
\end{insight}

We base our implementation on the code base of GR00T ~\cite{bjorck2025gr00t}. Three layer-fusion methods were implemented based on existing literatures: VLA-Adapter ~\cite{wang2025vlaadapter}, APT ~\cite{xu2026aptactionexpertpretraining}, and BEHAVIOR-1st ~\cite{larchenko2025taskadaptationvisionlanguageactionmodel}. APT adds each VLM layer onto the action stream through a weighted accumulative injection. VLA-adapter instead lets the action tokens cross-attend to each layer through the action head's cross-attention blocks. Both condition each action layer on a single VLM layer, with no cross-depth mixing. BEHAVIOR-1st is the only design that explicitly fuses layers: each action layer attends to a learned combination of all VLM layers, initialized to the standard layer-to-layer configuration.

Beyond fusion method, we are also interested in whether the performance of fusion is impacted by the choice of layers being fused. Instead of fusing all the layers, we only choose a subset of the layers and we define \textbf{layer configuration} as the strategy of choosing such a subset of layers. This will help us better understand how the choice of layers will impact downstream performance. More specifically, three layer configurations are implemented:
 (A) \emph{Even}, the layers are evenly selected with a constant stride ~\cite{sun2026rocketresidualorientedmultilayeralignment}. Setting the stride to 1 will match the exact interface setting of ~\cite{black2024pi0}. (B) \emph{Last}, the strategy that always chooses a constant number of deepest layers. This reflects the the reported advantages of deeper layers relative to other layers in ~\cite{yang2026abotm0vlafoundationmodel}. (C) \emph{Best}, the layers with the best single-layer performance under single-layer paradigm are chosen, which serves. Intuitively, the \emph{Best} layer configuration should serve as an ideal setting that outlines the upper-bound performance of each fusion method.

To isolate the impacts of layer-fusion, every feature other than fusion method and layer configuration is controlled to be identical as the GR00T infrastructure. For example, all methods are implemented with the same VLM base model. All methods share the same benchmark's training set and the same training pipeline based on flow-matching \cite{bjorck2025gr00t}. All evaluations are based on the same benchmark. We quantify the performance of each method with the overall success rate of the benchmark.
To obtain generalizable experimental conclusions, we repeat the above experiment settings with three base models (Cosmos-2B ~\cite{nvidia2025cosmos, kim2026contrastiverepresentationregularizationvisionlanguageaction}, Qwen-2B ~\cite{bai2025qwen3vltechnicalreport}, and Abot-Pretrain ~\cite{yang2026abotm0vlafoundationmodel}) and two benchmarks(LIBERO ~\cite{liu2023libero} and CALVIN ~\cite{mees2022calvinbenchmarklanguageconditionedpolicy}). 

Given a base model and a benchmark, each method under each layer configuration is given three trials of independent trainings with distinct random seeds. The main experiment result is reported in Table ~\ref{tab:fusion_vs_best}. 47 out of 54 possible combinations of fusion methods and layer configurations present performance lower than the single optimal layer as a single-layer interface. By its definition, the \emph{best} layer configuration should contain the optimal layer in the latent interface. However, 17 out of 18 cells in Table ~\ref{tab:fusion_vs_best} with \emph{best} layer configuration under-perform the single-the single optimal layer. This suggests that aggregating additional layers with the single optimal layers can provide low-quality noises instead of useful information.

There are 7 exceptional cases where fusion beats the single optimal layer in Table ~\ref{tab:fusion_vs_best}. They are discussed with greater details in Sec. ~\ref{sec:paradigm:exception}.

\begin{table*}[t]
\centering
\begin{threeparttable}
\caption{Layer fusion minus the single best layer, across benchmarks and base models}
\label{tab:fusion_vs_best}
\scriptsize
\setlength{\tabcolsep}{3pt}
\begin{tabular*}{\textwidth}{@{\extracolsep{\fill}}lccccccccc@{}}
\toprule
& \multicolumn{3}{c}{Cosmos-Reason2-2B} & \multicolumn{3}{c}{Qwen3-VL-2B} & \multicolumn{3}{c}{ABot-Pretrain-4B} \\
\cmidrule(lr){2-4} \cmidrule(lr){5-7} \cmidrule(lr){8-10}
Layer config & APT & VLA-Ad. & BEHAV. & APT & VLA-Ad. & BEHAV. & APT & VLA-Ad. & BEHAV. \\
\midrule
\multicolumn{10}{@{}l}{\textit{LIBERO}} \\
\quad best & $-1.8$ & $-3.4$  & $-4.5$  & $-4.1$ & $-4.1$  & $-5.4$  & $-2.1$ & $-1.1$ & $-0.1$ \\
\quad even & $-3.2$ & $-4.7$  & $-9.5$  & $-7.4$ & $-13.4$ & $-4.5$  & $-3.9$ & $-53.0$ & $-32.9$ \\
\quad last & $-3.6$ & $-14.9$ & $-20.2$ & $-8.0$ & $-21.1$ & $-26.7$ & $-1.4$ & $\mathbf{+0.2}$ & $\mathbf{+0.0}$ \\
\midrule
\multicolumn{10}{@{}l}{\textit{CALVIN}} \\
\quad best & $-11.8$ & $-5.0$  & $-5.2$  & $-5.7$  & $-5.3$  & $-7.0$  & $\mathbf{+9.9}$ & $-6.6$ & $-7.8$ \\
\quad even & $-14.7$ & $-8.8$  & $-5.3$  & $-16.1$ & $-6.8$  & $-6.8$  & $\mathbf{+7.4}$ & $-14.1$ & $-14.4$ \\
\quad last & $\mathbf{+2.9}$ & $-23.2$ & $-27.2$ & $\mathbf{+0.6}$ & $-20.9$ & $-26.1$ & $\mathbf{+20.7}$ & $-22.7$ & $-23.5$ \\
\bottomrule
\end{tabular*}
\begin{tablenotes}[flushleft]
\footnotesize
\item \textbf{Fusion cannot reliably beat the single optimal layer.} Each entry is (that fusion method's
success rate under that layer configuration) $-$ (the success rate of the optimal layer in Sec. ~\ref{sec:layer_choice_oracle}), in percentage points, both averaged over the same $3$ training
seeds. Negative means the single best layer wins. The exceptions where the single best layer loses are discussed in ~\ref{sec:paradigm:exception}.
\end{tablenotes}
\end{threeparttable}
\end{table*}

\subsection{Statistical Analysis of Fusion's Advantage}

Table ~\ref{tab:fusion_vs_best} considers the average performance difference while ignoring the performance variation of each fusion methods or single optimal layer. Appendix ~\ref{app:stats_fusion_analysis} mitigates this gap. We conduct one-sided Welch’s two-sample \(t\)-test ~\cite{welch1947generalization}to rigoroursly verify the advantage of each fusion method over the single optimal layer while considering the performance variations. The main result is reported in Table ~\ref{tab:fusion_pvalues}. The advantage of each fusion method is coupled with very low statistical confidence while APT with CALVIN benchmark is the only exception.

\subsection{Exception Discussion}
\label{sec:paradigm:exception}
In Table ~\ref{tab:fusion_vs_best}, 5 of the 7 fusion-win exceptions use APT as the fusion method. There are two changing factors nested in the comparison between APT and GR00T's single-layer policy: (A) APT introduces an action-head parameterization different than the default one of GR00T; such architecture change benefits downstream performance. (B) APT uses a fusion tuple of multiple layers as input instead of a single layer; this fusion benefits downstream performance. Our experiment in Appendix ~\ref{app:exception_analysis} decouples these two factors. The advantage of APT observed in Table ~\ref{tab:fusion_vs_best} can be obtained with only the architecture-change factor (A) while excluding the fusion factor. Figure ~\ref{fig:apt_repeat_vs_groot} reports this fact. This suggests that architecture change is the major contributor instead of fusion.

\section{Layer Selection for Single-layer Paradigm}
\label{Sec:layer_choice}

\subsection{Layer-Quality Oracle}
\label{sec:layer_choice_oracle}

\begin{insight}
Naive Layer-selection with benchmark trainings and evaluations across every layer is computationally expensive and requires a non-trivial cheap proxy.
\end{insight}
This section specifies how we estimate the oracle quality of each layer, which utilizes computationally-expensive estimation scheme. A decent layer-selection scheme should have metrics that are a cheap proxy for such oracle quality.

Given a benchmark and a VLM backbone with $D$ layers, we sweep through the layer indices $\{1, \cdots, D \}$. For each layer with index $i$, we train the action head by taking layer $i$ as a single-layer latent interface using a benchmark's training set. \emph{We treat the benchmark evaluation success rate of the trained VLA network as the oracle quality of such layer $i$.} In this work, two benchmarks are adopted: LIBERO and CALVIN. All other engineering choices other than the layer number are controlled to be identical. For each layer $i$, the above training and evaluation are repeated with $3$ random seeds to remove the effects of noise. To make our results more base-model robust, we repeat layer-sweep experiment above with three VLM backbones: Cosmos-Reason-2B, Qwen-2B, and Abot-pretrain. We further define the optimal layer as the single layer with the highest oracle quality.
 
\begin{table*}[t]
\centering
\begin{minipage}[t]{0.44\textwidth}
\centering
\footnotesize
\begin{threeparttable}
\caption{GPU-Hours}
\label{tab:layer_sweep_cost}
\setlength{\tabcolsep}{3pt}
\begin{tabular}{@{}lccc@{}}
\toprule
& \shortstack{Cosmos-2B\\(28L)} & \shortstack{Qwen-2B\\(28L)} & \shortstack{Abot-4B\\(36L)} \\
\midrule
\multicolumn{4}{@{}l}{\textit{LIBERO}} \\
\quad Oracle  & 129.1 & 127.4 & 152.5 \\
\quad InfoNCE &  10.6 &  14.1 &   4.6 \\
\midrule
\multicolumn{4}{@{}l}{\textit{CALVIN}} \\
\quad Oracle  & 119.1 &  94.0 & 133.7 \\
\quad InfoNCE &   7.4 &   7.5 &  13.4 \\
\bottomrule
\end{tabular}
\begin{tablenotes}[flushleft]
\footnotesize
\item \textbf{InfoNCE sweep is an order of magnitude cheaper than oracle sweep}. Each entry is total GPU-hours across the $3$ seeds of that sweep. One GPU-hour is one H100 device for one hour. See Sec.~\ref{sec:layer_choice_oracle} for the definition of oracle sweep.
\end{tablenotes}
\end{threeparttable}
\end{minipage}\hfill
\begin{minipage}[t]{0.54\textwidth}
\centering
\footnotesize
\begin{threeparttable}
\caption{Optimal layer per benchmark and base model}
\label{tab:optimal_layer}
\setlength{\tabcolsep}{3pt}
\begin{tabular}{@{}lccc@{}}
\toprule
& \shortstack{Cosmos-2B\\(28L)} & \shortstack{Qwen-2B\\(28L)} & \shortstack{Abot-4B\\(36L)} \\
\midrule
\multicolumn{4}{@{}l}{\textit{LIBERO}} \\
\quad Optimal layer $l^\star$        & 15   & 3    & 31   \\
\quad Succ. at $l^\star$ (\%)      & 76.3 & 76.8 & 98.4 \\
\quad Succ. at deepest (\%)  & 62.9 & 55.6 & 96.6 \\
\midrule
\multicolumn{4}{@{}l}{\textit{CALVIN}} \\
\quad Optimal layer $l^\star$        & 11   & 16   & 12   \\
\quad Succ. at $l^\star$ (\%)      & 64.4 & 64.7 & 66.1 \\
\quad Succ. at deepest (\%)  & 39.6 & 36.0 & 48.8 \\
\bottomrule
\end{tabular}
\begin{tablenotes}[flushleft]
\footnotesize
\item \textbf{The optimal layer is never the deepest layer, and its depth is not transferable
across backbones or benchmarks.} The optimal layer $l^\star$ is the layer with the highest oracle
quality (Sec.~\ref{sec:layer_choice_oracle})
\end{tablenotes}
\end{threeparttable}
\end{minipage}
\end{table*}
For the rest of this paper, we will refer to such layer-sweep experiment with full-scale training and evaluation as \emph{ oracle sweep}. It is expensive in terms of both time and computation. We report the computational cost of oracle sweep against infoNCE sweep, a proxy of oracle sweep to be discussed later, in terms of GPU-Hour in Table ~\ref{tab:layer_sweep_cost}. The heavy expense partially originates from the fact that most layers present large cross-seed variance, which makes multiple-seed trials necessary for each layer(See Fig. ~\ref{fig:success_vs_layer}). Note that oracle sweep is still itself a proxy for each layer's ground-truth quality since our evaluation of each layer's trained policy only resides in simulation environment. The ground-truth quality of each layer involving hardware estimation will be even more expensive than Table ~\ref{tab:layer_sweep_cost} suggests. This further motivates the significance of a cheap proxy scheme that can accurately locate the optimal layer.

Table~\ref{tab:optimal_layer} reports
the optimal layer for every benchmark--base-model pair, together with the success rate the same
backbone attains at its deepest layer. We present two observations. First, the optimal layer is shallower than the deepest
layer. Thus, the selection of the optimal layer can offer us faster inference speed by discarding unnecessary layers beyond the optimal one. This further confirms the benefits of a layer-selection scheme. Second, the range of the optimal layer varies widely without any pattern that is easily recognizable. Switching the benchmark from LIBERO to CALVIN drifts the optimal layer of Qwen-2B from 3 to 16. Given the same LIBERO benchmark, Abot prefers the deeper layers(31 out of 36) while Cosmos-2B prefers the mid layers(15 out of 28). This further suggest that the selection of VLA's optimal layer is highly non-trivial.

\subsection{InfoNCE as Layer-Quality Proxy}
\label{sec:infonce_theory}
\begin{insight}
InfoNCE and action prediction are two learning tasks whose information objectives are defined with the same quantities.
\end{insight}
Oracle sweep is an iteration whose unit-level process is an estimation of quality for one layer, which we refer to as oracle quality estimation(Sec. ~\ref{sec:layer_choice_oracle}). Thus, a cheap proxy metrics for one layer's oracle quality is necessary for a layer-selection scheme that can replace oracle sweep. We prepare a list candidates of such metrics:
infoNCE ~\cite{oord2019representationlearningcontrastivepredictive} , DiME ~\cite{skean2023dimemaximizingmutualinformation}, Prompt-level Matrix Entropy ~\cite{skean2025layer}, Dataset-level Matrix Entropy ~\cite{skean2025layer}.
Note that each metrics consumes much lower computing power because they are either training-free or training-light. The GPU-Hour to sweep layers for infoNCE estimation is reported in Table ~\ref{tab:layer_sweep_cost}, which is of 10-times smaller than oracle sweep.

To empirically verify how each metric correlates with oracle quality, we present the following experiment. Given a VLM backbone, we sweep through its layers and measure the metrics value for each layer, which will present us a Metrics Value v.s. Layer Depth curve for each metric. Then we compute the Spearman correlation between such curve and the Oracle Quality v.s. Layer Depth curve obtained from the oracle sweep. The main result is reported in Table ~\ref{tab:metric_correlation}.

InfoNCE turns out to be the most reliable predictor. To theoretically interpret the predictive power of infoNCE, consider the following theorem:

\begin{theorem}[IB objective Equivalence]
\label{Thm:IB_eq}
Recall the IB objective for action prediction(Prop. ~\ref{prop1:info_action_objective}) and the IB objective for infoNCE(Prop. ~\ref{prop2:infonce_objective}). If $\beta < \frac{1}{N-1}$, then
\begin{equation}
\begin{aligned}
&\max_{\phi} I(Z';A', Y) - \beta \; I(Z'; X' |A', Y) \\
\equiv\;  & \max_{\phi} I(Z, A) - \frac{\beta \; N}{1 - \beta \; (N-1)} \; I(Z, X|A)
\end{aligned}
\end{equation}
\end{theorem}

\begin{proof}
See Appendix ~\ref{appendix:thm_ib_eq_proof}.
\end{proof}

Theorem ~\ref{Thm:IB_eq} is arguing that the IB objective for infoNCE (Prop. ~\ref{prop2:infonce_objective}) is a reweighted version of the IB objective of action prediction(Prop. ~\ref{prop1:info_action_objective}). Both IB objectives are defined by the same information-theoretical quantities $I(Z, A)$ and $I(Z, X|A)$.

Our Theorem gives InfoNCE a different reading from previous works ~\cite{oord2019representationlearningcontrastivepredictive, poole2019variationalboundsmutualinformation}. Previous researchers view infoNCE as an lower-bound estimation for mutual information
\begin{equation}
\label{eq:infoNCE_old_view}
I(Z, A) \geq \log N - \mathcal{L}_{\text{infoNCE}}(\phi, h)
\end{equation}

Remarkably, such bound becomes tight when $N$ is sufficiently large ~\cite{pmlr-v139-sordoni21a} ~\cite{oord2019representationlearningcontrastivepredictive}. Appendix ~\ref{app:bound_analysis} empirically confirms the tightness of bound ~\ref{eq:infoNCE_old_view} under the interpretation of ~\cite{oord2019representationlearningcontrastivepredictive}. However, such perspective cannot cleanly explain infoNCE's strong correlation with oracle quality.
For the $l$th-layer latent and its subsequent layer of a VLM, we have $I(Z_l, A) \geq I(Z_{l+1}, A)$, which immediately follows by the determinism of all layers and data processing theorem. That means, given the tightness of bound ~\ref{eq:infoNCE_old_view}, we would always find $\log N - \mathcal{L}_{\text{infoNCE}}(\phi_l, h)$ monotonically decreases with layer number $l$. In our experiment, we find this to be not true. In fact, we find that $\log N - \mathcal{L}_{\text{infoNCE}}(\phi_l, h)$ can closely track the trend of Oracle Quality v.s. layer number which is NOT monotonically decreasing(Fig. ~\ref{fig:success_vs_infoNCE_grid}). This suggests some imperfectness in the optimization of infoNCE loss, and more importantly, such imperfection is shared by both infoNCE optimization and VLA training. Prop. ~\ref{prop2:infonce_objective} and Theorem ~\ref{Thm:IB_eq} provide an information-theoretical lens to interpret such imperfection. Information bottleneck frames these imperfections as task-irrelevant noises which can be reweighted to be equivalent for action prediction and infoNCE. We note that there might be other non information-theoretical reasons for such imperfectness.

\begin{figure*}[t]
  \centering
  \includegraphics[width=\linewidth]{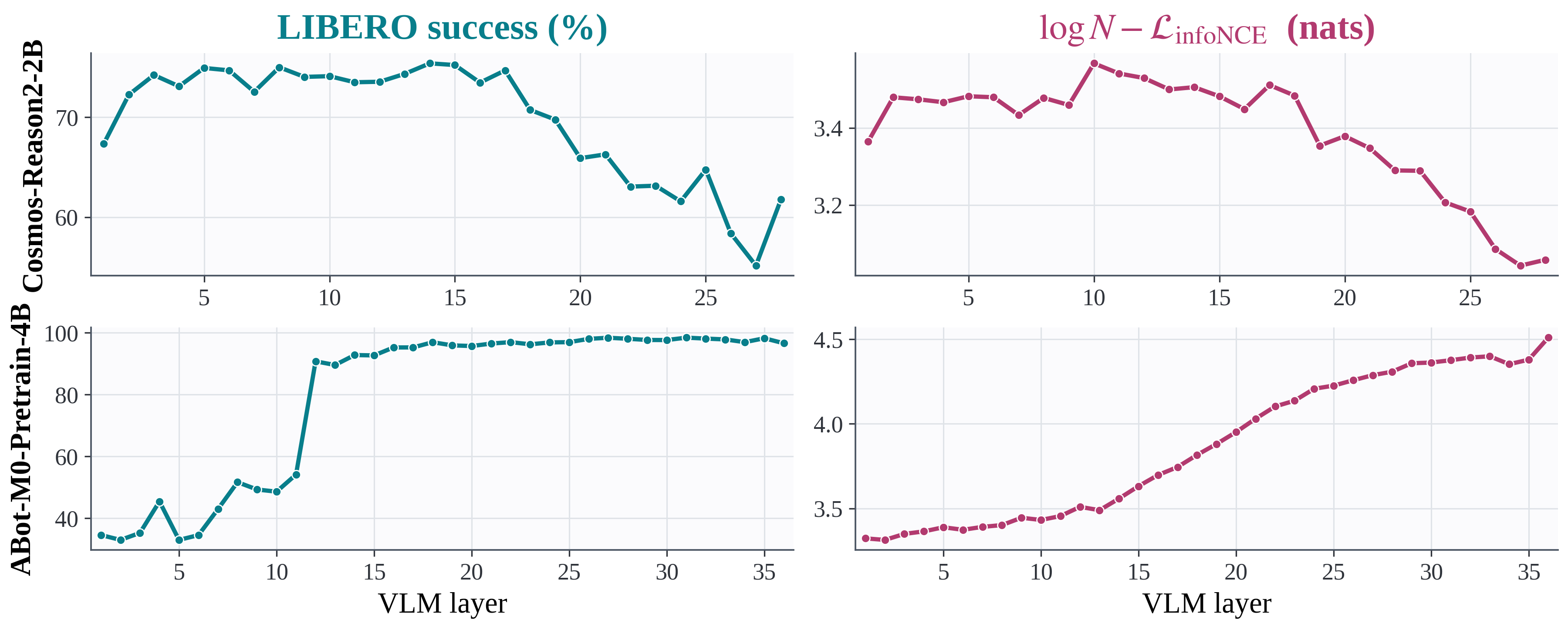}
  \caption{\textbf{The $\log N - \mathcal{L}_{\text{infoNCE}}$-versus-layer curve closely tracks the non-monotone LIBERO-success-versus-layer curve.} The left column shows LIBERO success and the right column shows InfoNCE for Cosmos-2B (top) and Abot-Pretrain (bottom).}
\label{fig:success_vs_infoNCE_grid}
\end{figure*}

\begin{table*}[t]
\centering
\footnotesize
\begin{threeparttable}
\caption{Spearman correlation of proxy metrics with oracle success across layers}
\label{tab:metric_correlation}
\setlength{\tabcolsep}{3pt}
\renewcommand{\arraystretch}{1.15}
\begin{tabular*}{\textwidth}{@{\extracolsep{\fill}}lccc@{}}
\toprule
& \shortstack{Cosmos-2B\\(28L)} & \shortstack{Qwen-2B\\(28L)} & \shortstack{Abot-4B\\(36L)} \\
\midrule
\multicolumn{4}{@{}l}{\textit{LIBERO}} \\
InfoNCE                      & $\mathbf{+0.82}$ & $\mathbf{+0.85}$ & $\mathbf{+0.94}$ \\
DiME                         & $+0.72$          & $+0.62$          & $+0.85$ \\
Prompt-level entropy  & $-0.82$          & $-0.79$          & $-0.03$ \\
Dataset-level entropy & $-0.69$          & $-0.53$          & $+0.77$ \\
\midrule
\multicolumn{4}{@{}l}{\textit{CALVIN}} \\
InfoNCE                      & $\mathbf{+0.82}$ & $\mathbf{+0.72}$ & $+0.36$ \\
DiME                         & $+0.22$          & $+0.26$          & $+0.22$ \\
Prompt-level entropy         & $-0.41$          & $-0.28$          & $-0.75$ \\
Dataset-level entropy        & $-0.41$          & $-0.15$          & $\mathbf{-0.83}$ \\
\bottomrule
\end{tabular*}
\begin{tablenotes}[flushleft]
\footnotesize
\item \textbf{InfoNCE is the most consistent predictor across benchmarks and backbones.} Each cell reports the Spearman rank correlation $\rho$ between a metric's value-versus-layer curve and the oracle-quality-versus-layer curve. Bold marks the highest $|\rho|$ within each benchmark--backbone column. A larger $|\rho|$ indicates that the metric more reliably tracks policy success across layers.
\end{tablenotes}
\end{threeparttable}
\end{table*}

\subsection{Layer-Selection Scheme}
\label{sec:layer-selection-scheme}
\begin{insight}
InfoNCE-argmax selector can come close to the performance of the heuristics selector optimized with oracle data.
\end{insight}
We now test whether the predictive signal identified in Section~4.2 translates into effective layer selection.

Given our definition of the optimal layer and our insights of infoNCE, it follows naturally to design the layer selection scheme as an infoNCE argmax selector. Given a
benchmark and a base model, we sweep through layer indices $\{1, \cdots, D\}$. For each layer, we train infoNCE critic $h$ using the benchmark's data with three random seeds. The selected layer is the one maximizing the seed-averaged InfoNCE
estimate, $\hat{l} = \arg\max_l \{ \log N - \mathcal{L}_{\text{infoNCE}}(Z_l) \}$ where $\mathcal{L}_\text{infoNCE}(Z_l)$ is the infoNCE loss after optimizing the critic for layer $l$. We will refer to such scheme as infoNCE-argmax.

Wo compare the regrets infoNCE with different constant-layer heuristics in Table ~\ref{tab:layer_selection_regret}. Here, we define regrets as the performance difference between the layer selected by the scheme and the optimal layer. By constant-layer heuristics, we mean the layer-selection strategy that chooses layer as a constant number or relative depth. The last-layer heuristics reflects the arguments of ~\cite{yang2026abotm0vlafoundationmodel}.
We also include some two oracle constant-layer heuristics: one with the constant layer number optimized for minimum average regrets over all 6 settings, and one optimized for minimum max regrets over all 6 settings. These two oracle heuristics should serve as an upper-bound of layer-selection schemes because of their supreme knowledge of all the oracle layer qualities.

\begin{table*}[t]
\centering
\footnotesize
\begin{threeparttable}
\caption{Regret of layer-selection schemes}
\label{tab:layer_selection_regret}
\setlength{\tabcolsep}{3pt}
\renewcommand{\arraystretch}{1.12}
\begin{tabular*}{\textwidth}{@{\extracolsep{\fill}}lccc@{}}
\toprule
Regret (success \%) & \shortstack{Cosmos-2B\\(28L)} & \shortstack{Qwen-2B\\(28L)} & \shortstack{Abot-4B\\(36L)} \\
\midrule
\multicolumn{4}{@{}l}{\textit{LIBERO}} \\
InfoNCE-argmax           & 2.1 (L10) & 7.3 (L10) & 1.8 (L36) \\
Deepest layer            & 13.4 (L28) & 21.3 (L28) & 1.8 (L36)\\
Oracle constant, min-mean (L16)    & 4.0 & 8.3 & 3.2 \\
Oracle constant, min-max (L15)     & 0.0 & 6.4 & 5.8 \\
\midrule
\multicolumn{4}{@{}l}{\textit{CALVIN}} \\
InfoNCE-argmax           & 0.0 (L11) & 2.2 (L14) & 8.8 (L20) \\
Deepest layer            & 24.8 (L28) & 28.7 (L28) & 17.2 (L36) \\
Oracle constant, min-mean (L16)    & 2.1 & 0.0 & 2.0 \\
Oracle constant, min-max (L15)     & 2.7 & 4.0 & 2.1 \\
\bottomrule
\end{tabular*}
\begin{tablenotes}[flushleft]
\footnotesize
\item \textbf{InfoNCE-argmax can come close to the performance of the heuristics selector optimized based on oracle data.} Each non-fusion cell reports regrets in success percentage points; the selected layer is in parentheses. Regret = performance of optimal layer - performance of the selected layer. Constant-layer heuristics tap the same index in every setting. The two oracle-constant heuristics chooe that index by minimizing, over $l\in\{1,\dots,28\}$, either the mean regret across the six settings (L16: $3.28$ mean, $8.28$ worst case) or the worst-case regret (L15: $6.35$ worst case, $3.50$ mean).
\end{tablenotes}
\end{threeparttable}
\end{table*}

InfoNCE-argmax improves over deepest-layer selection in five of the six backbone--benchmark settings and ties it in the sixth. Averaged equally across these settings, its regret is $3.71$ points, compared with $17.89$ points for the deepest layer. Its proxy sweep also requires approximately $9$--$33$ times less GPU compute than the oracle sweep (Table ~\ref{tab:layer_sweep_cost}). It also comes close to the performance of the two oracle heuristics. The two oracle heuristics have average regrets of 3.28 percentage points and 3.50 comparing to 3.71 of infoNCE-argmax.

\section{Conclusion}
We studied how layer selection and layer fusion shape the latent interface of VLA policies with frozen backbones. Our experiments validate the limitation of fusion methods. For the single-layer paradigm, the best layer varies substantially across backbones and benchmarks, making efficient layer selection necessary. We derive a reweighting equivalence between our proposed information-bottleneck objectives for InfoNCE and action prediction, providing a theoretical motivation for using InfoNCE as a layer-quality proxy. We further empirically test the performance of infoNCE-argmax layer selection scheme.

Our results establish layer selection as a consequential design choice and InfoNCE as a practical tool for making that choice in the evaluated settings. Future work should examine how action-head capacity, backbone fine-tuning, and real-world deployment affect layer quality and the benefits of fusion.

\bibliography{ref}
\bibliographystyle{IEEEtran}

\appendix

\section{Theory Appendix}
\subsection{Proposition Reasoning}
\label{Appendix:defend_prop1}
The only difference between the optimization stated in Formula One of ~\cite{alemi2019deepvariationalinformationbottleneck} and our proposition ~\ref{prop1:info_action_objective} is the constraint term. We state $I(Z, X | A) \leq I_c$ while ~\cite{alemi2019deepvariationalinformationbottleneck} states $I(Z, X) \leq I_c$. Note that both constructions are alternative representations of information bottleneck. Such a switch is exactly the core argument made by ~\cite{kawaguchi2023doesinformationbottleneckhelp} who argues that $I(Z, X |A)$ can be used to construct a much tighter bound on generalization error than $I(Z, X)$.

\subsection{Theoretical Reasoning of Layer-Fusion}
\label{Sec:paradigm_choice_theory}

This section aims to theoretically reason about our observations in Sec. ~\ref{Sec:paradigm_choice_exp}. Note that each fusion method in Sec. ~\ref{Sec:paradigm_choice_exp} designs their fusion mechanism as an optimizable module of the action head. Concretely, APT models fusion of multiple layers as a learnable gate of the action head, while BEHAVIOR-1st's combination weights are also optimizable during the training of the action head. Thus, it is reasonable to interpret their latent interface as a concatenation of the selected layers while the fusing operations are part of the action head. With that in mind, consider the following theorem:

\begin{theorem}
\label{Thm:layer_concat}[Shallow-Layer Equivalence]
Suppose a VLM backbone of $D$ layers has layer-indices $\{1, 2, 3, \cdots,D \}$. Let $\mathcal{C}$ be the set of selected layers, i.e. $\mathcal{C} \subset \{1, 2, 3, \cdots,D \}$. Let $Z_{\mathcal{C}}$ be a tuple concatenating layer outputs specified by $\mathcal{C}$, $Z_{\mathcal{C}} = (Z_l)_{l \in  \mathcal{C}}$. Let $l^*$ be the smallest element of $\mathcal{C}$, i.e. $l^* = \min_{l\in\mathcal{C}} l$. Then,

\begin{equation}
I(Z_{\mathcal{C}}, A) = I(Z_{l^*}, A), \; I(Z_{\mathcal{C}}, X |A) = I(Z_{l^*}, X |A)
\end{equation}
\end{theorem}

\begin{proof}
See Appendix. ~\ref{appendix:thm_layer_concat_proof}.
\end{proof}

Recall from Prop. ~\ref{prop1:info_action_objective} that we have two information quantities of interest for a given latent interface: $I(Z, A)$ and $I(Z, X | A)$. Theorem ~\ref{Thm:layer_concat} argues that a tuple of layers, $Z_{\mathcal{C}}$, holds these two quantities exactly the same as the single shallowest layer in the tuple, $Z_{l^*}$. Thus, fusion of a set of layers is information-theoretically equivalent as the shallowest layer in the set, which resolves why fusion cannot beat the single optimal layer under most situations in Table ~\ref{tab:fusion_vs_best}.

In Table ~\ref{tab:fusion_delta_all}, we report the difference between each fusion method's performance and that of the single shallowest layer. In roughly half of the cells(26 out of 54) in Table ~\ref{tab:fusion_delta_all}, fusion cannot beat the shallowest layer, which is consistent with Theorem ~\ref{Thm:layer_concat}. The other half of the cells in Table ~\ref{tab:fusion_delta_all}, where fusion beats the its shallowest constituent, suggests that fusion can still help downstream performance from perspectives other than information theory. 

Prop.~\ref{prop1:info_action_objective} only captures how much information is encoded in the latent interface. They do not capture the decodability of the information ~\cite{xu2020theoryusableinformationcomputational}. Fusion may still help through non information-theoretical properties, such as the geometry of the latent space, its linear decodability, or its interaction with the action head's optimization dynamics. Whether fusion offers any advantage along these axes remains open. Our research should further motivate the community to re-consider the roles of fusion instead of denying fusion completely. As part of our future work, we will further extend our theoretical framework to consider the limited functional capacity of the action head following ~\cite{xu2020theoryusableinformationcomputational}.

\subsection{Proof for Theorem ~\ref{Thm:layer_concat}}
\label{appendix:thm_layer_concat_proof}
\begin{proof}
Write $\mathcal{C} = \{l^* = l_1 < l_2 < \cdots < l_k\}$. For any $l \geq l^*$,
define the partial composition
$\psi_{l^* \to l} := \psi_{l} \circ \psi_{l-1} \circ \cdots \circ \psi_{l^*+1}$,
with the convention $\psi_{l^* \to l^*} := \mathrm{id}$. By the assumption in
Sec.~\ref{sec:formulation} that every backbone layer $\psi_i$ is a
deterministic function, each $\psi_{l^* \to l}$ is a deterministic measurable
map, and since $Z_l = \psi_l \circ \cdots \circ \psi_1(X)$ we have
$Z_l = \psi_{l^* \to l}(Z_{l^*})$ for every $l \in \mathcal{C}$. Consequently
$Z_{\mathcal{C}} = h(Z_{l^*})$, where
$h(z) := \big(\psi_{l^* \to l}(z)\big)_{l \in \mathcal{C}}$. Conversely, letting
$\pi$ denote the coordinate projection onto the $l^*$ component of the tuple, we
have $Z_{l^*} = \pi(Z_{\mathcal{C}})$, because $l^* \in \mathcal{C}$ and
$\psi_{l^* \to l^*} = \mathrm{id}$. Thus $h$ and $\pi$ are two valid deterministic mappings. Based on this, we can conclude:
\begin{itemize}
\item $A \to Z_{l^*} \to Z_{\mathcal{C}}$ is a Markov chain and the data
processing inequality yields
$I(Z_{\mathcal{C}}; A) = I(h(Z_{l^*}); A) \leq I(Z_{l^*}; A)$;

\item $A \to Z_{\mathcal{C}} \to Z_{l^*}$ is a Markov chain and gives
$I(Z_{l^*}; A) = I(\pi(Z_{\mathcal{C}}); A) \leq I(Z_{\mathcal{C}}; A)$.

\end{itemize}
Combining the two inequalities establishes
$I(Z_{\mathcal{C}}; A) = I(Z_{l^*}; A)$.

Then, observe that
the maps $h$ and $\pi$ do not depend on $A$, so conditioning on $A = a$ leaves
the relations $Z_{\mathcal{C}} = h(Z_{l^*})$ and $Z_{l^*} = \pi(Z_{\mathcal{C}})$
intact for $P_A$-almost every $a$. Applying the conditional data processing
inequality to the conditional Markov chains $X \to Z_{l^*} \to Z_{\mathcal{C}}$
and $X \to Z_{\mathcal{C}} \to Z_{l^*}$ given $A = a$ yields
$I(Z_{\mathcal{C}}; X \mid A = a) = I(Z_{l^*}; X \mid A = a)$, and taking the
expectation over $a \sim P_A$ gives
$I(Z_{\mathcal{C}}; X \mid A) = I(Z_{l^*}; X \mid A)$.
\end{proof}

\subsection{Alternative Construction of Prop. ~\ref{prop2:infonce_objective}}
\label{app:alt_prop2_construct}
In this section, we discuss about alternative constructions of Prop. ~\ref{prop2:infonce_objective} and motivate why ~\ref{prop2:infonce_objective} is more reasonable. Recall that an information bottleneck objective involves two construction choices: task-relevant signal and task-irrelevant noise. 

One alternative choice of classification-relevant signal is $I(Z', Y |A')$. The quantity $I(Z';Y\mid A')$ measures the information supplied by the latent tuple about classification when the action is given, whereas $I(Z';A',Y)=I(Z';A')+I(Z';Y\mid A')$ additionally accounts for action information retained by the tuple. Our use of $I(Z';A', Y)$ is motivated by the practical preference of contrastive training to compact information like $A'$. Classification training can retain information that does not affect its label-classification ~\cite{teterwak2021understanding}. 
 Studies of contrastive learning show that optimization bias can suppress features and favor low-rank representations \cite{xue2023features,jing2022collapse}. These results motivate a preference for compact conditional information like $A'$ that is classification-irrelevant.

One alternative construction of task-irrelevant noise is $I(Z'; X', A'|Y)$. This quantity measures the the joint information of $X'$ and $A'$ independent of $Y$. This construction assumes that a desirable $Z'$ should discard classification-irrelevant information in $A'$. We reject this construction because of the same reason as the classification-relevant term: $A'$, as a compact conditional signal, is preferred to be preserved by the training dynamics of contrastive learning.

\subsection{Proof for Theorem ~\ref{Thm:IB_eq}}
\label{appendix:thm_ib_eq_proof}
\begin{proof}
Since $\phi$ is deterministic,
$Z=\phi(X)$ gives $I(Z;A\mid X)=0$, hence by the chain rule
\begin{equation}
\label{eq:star}
I(Z;X)=I(Z;A)+I(Z;X\mid A). \tag{$\star$}
\end{equation}

\emph{Relevance term.} By the chain rule,
$I(Z';A',Y)=I(Z';Y)+I(Z';A'\mid Y)$. The distribution of $Z'$
is exchangeable in its coordinates and does not depend on which index is the
positive one, so $Z'\perp Y$ and $I(Z';Y)=0$.

Given
$(A',Y=y)$, the tuple $Z'$ splits into the positive sample $Z_y$, whose law
depends on $A'$, and the negative samples $Z_{\neq y}:=(Z_i)_{i\neq y}$, who is independent of both $A'$ and $Z_y$. In particular
\begin{equation}
\label{eq:neg_indep}
Z_{\neq y}\;\perp\;(Z_y,A')\quad\text{given } Y=y.
\end{equation}

Now apply the chain rule for mutual information to the tuple
$Z'=(Z_y,Z_{\neq y})$, conditioning throughout on $Y=y$:
\begin{equation}
\label{eq:chain_split}
\begin{aligned}
&I\big(Z';A'\mid Y=y\big) \\
= &I\big(Z_y;A'\mid Y=y\big)
+ I\big(Z_{\neq y};A'\mid Z_y,\,Y=y\big).
\end{aligned}
\end{equation}
By \eqref{eq:neg_indep},
$I(Z_{\neq y};A'\mid Z_y,Y=y)=0$. Therefore
\begin{equation}
\label{eq:reduce_to_pos}
I\big(Z';A'\mid Y=y\big)=I\big(Z_y;A'\mid Y=y\big).
\end{equation}

Conditioned on $Y=y$, the positive pair is generated by
$(X_y,A')\sim P_{X,A}$ followed by $Z_y=\phi(X_y)$; this is exactly the law of
$(Z,A)$ obtained by pushing $(X,A)\sim P_{X,A}$ through $\phi$. Moreover, the
value $y$ enters only as a coordinate relabeling and does not affect the joint
law of the positive pair, so $I(Z_y;A'\mid Y=y)$ does not depend on $y$ and
equals its unconditional counterpart:
\begin{equation}
\label{eq:pos_is_ZA}
I\big(Z_y;A'\mid Y=y\big)=I(Z;A).
\end{equation}
Combining \eqref{eq:reduce_to_pos} and \eqref{eq:pos_is_ZA} gives
$I(Z';A'\mid Y=y)=I(Z;A)$ for every $y\in\{1,\dots,N\}$.

Finally, mutual information conditioned on the random variable $Y$ is by
definition the average of the conditional-on-$\{Y=y\}$ values,
\begin{equation}
\label{eq:average_out_Y}
\begin{aligned}
& I\big(Z';A'\mid Y\big) \\
= &\sum_{y=1}^{N} P(Y=y)\,I\big(Z';A'\mid Y=y\big) \\
= &\sum_{y=1}^{N} \frac{1}{N}\,I(Z;A)
= I(Z;A),
\end{aligned}
\end{equation}

Together with $I(Z';Y)=0$ and the chain-rule expansion
$I(Z';A',Y)=I(Z';Y)+I(Z';A'\mid Y)$, this yields
\begin{equation}
\label{eq:relevance}
I(Z';A',Y)=I(Z;A).
\end{equation}

\emph{Compression term.} Given
$A'$ and $Y=y$, the coordinates $X_1,\dots,X_N$ are mutually independent
($X_y\sim P_{X\mid A'}$, $X_i\sim P_X$ for $i\neq y$), so the pairs
$(Z_i,X_i)$ form independent blocks and mutual information is additive across
them:
\begin{equation}
\label{eq:info-add}
\begin{aligned}
&I(Z';X'\mid A',Y=y) \\
& =\sum_{i=1}^N I(Z_i;X_i\mid A', Y=y) \\
& = \{ \sum_{i\neq y} I(Z_i;X_i\mid A', Y=y) \} + I(Z_y; X_y|A', Y=y)
\end{aligned}
\end{equation} The positive coordinate
contributes $I(Z_y;X_y\mid A',Y=y)=I(Z;X\mid A)$ since $(Z_y,X_y\mid A')\sim
P_{Z,X\mid A}$. Each negative coordinate has $X_i\perp A'$, so conditioning on
$A'$ is vacuous and $I(Z_i;X_i\mid A')=I(Z_i;X_i)=I(Z;X)$. Summing the one
positive and $N-1$ negative contributions and applying $(\star)$ to the
latter,
\begin{align}
I(Z';X'\mid A',Y)
&= I(Z;X\mid A)+(N-1)\,I(Z;X) \notag\\
&= N\,I(Z;X\mid A)+(N-1)\,I(Z;A).
\label{eq:compression}
\end{align}
This value is independent of $y$, so averaging over $Y$ leaves it unchanged.

\emph{Reduction of the objective.} Substituting \eqref{eq:relevance} and
\eqref{eq:compression} into the InfoNCE IB objective of
Prop.~\ref{prop2:infonce_objective},
\begin{align}
&\max_{\phi}\; I(Z';A',Y)-\beta\,I(Z';X'\mid A',Y) \notag\\
={}&\max_{\phi}\; I(Z;A)-\beta\big[N\,I(Z;X\mid A)+(N-1)\,I(Z;A)\big] \notag\\
={}&\max_{\phi}\; \big[1-\beta(N-1)\big]\,I(Z;A)-\beta N\,I(Z;X\mid A).
\label{eq:collected}
\end{align}
By hypothesis $\beta<\tfrac{1}{N-1}$, so the coefficient
$1-\beta(N-1)$ is strictly positive. Dividing the objective in
\eqref{eq:collected} by this positive constant does not change the maximizing
$\phi$, giving the equivalent problem
\begin{equation}
\max_{\phi}\; I(Z;A)-\frac{\beta N}{1-\beta(N-1)}\,I(Z;X\mid A).
\end{equation}
Since $1-\beta(N-1)>0$, the multiplier
$\tfrac{\beta N}{1-\beta(N-1)}$ is well-defined and positive, and the two
maximization problems share the same set of maximizers, which is the claimed
equivalence with $\alpha=\tfrac{\beta N}{1-\beta(N-1)}$ in
Prop.~\ref{prop1:info_action_objective}.
\end{proof}

\begin{table*}[t]
\centering
\begin{threeparttable}
\caption{Layer fusion minus its shallowest constituent layer, across benchmarks and base models}
\label{tab:fusion_delta_all}
\scriptsize
\setlength{\tabcolsep}{3pt}
\begin{tabular*}{\textwidth}{@{\extracolsep{\fill}}lccccccccc@{}}
\toprule
& \multicolumn{3}{c}{Cosmos-Reason2-2B} & \multicolumn{3}{c}{Qwen3-VL-2B} & \multicolumn{3}{c}{ABot-Pretrain-4B} \\
\cmidrule(lr){2-4} \cmidrule(lr){5-7} \cmidrule(lr){8-10}
Layer config & APT & VLA-Ad. & BEHAV. & APT & VLA-Ad. & BEHAV. & APT & VLA-Ad. & BEHAV. \\
\midrule
\multicolumn{10}{@{}l}{\textit{LIBERO}} \\
\quad best & $+7.2$ & $-2.0$ & $-3.1$  & $+2.2$ & $-2.2$ & $-3.6$ & $-0.6$ & $-0.7$ & $+0.3$ \\
\quad even & $+5.8$ & $+4.3$ & $-0.6$  & $-1.1$ & $-7.1$ & $+1.8$ & $+60.0$ & $+10.9$ & $+31.0$ \\
\quad last & $-1.7$ & $-4.8$ & $-10.2$ & $-3.2$ & $-2.5$ & $-8.1$ & $+0.5$ & $+1.0$ & $+0.8$ \\
\midrule
\multicolumn{10}{@{}l}{\textit{CALVIN}} \\
\quad best & $+5.7$ & $-0.9$ & $-1.2$  & $+14.4$ & $-1.4$ & $-3.0$ & $+20.6$ & $+0.8$ & $-0.5$ \\
\quad even & $+6.4$ & $+12.3$ & $+15.8$ & $+6.0$ & $+15.3$ & $+15.3$ & $+28.6$ & $+7.1$ & $+6.8$ \\
\quad last & $+6.7$ & $-4.7$ & $-8.7$  & $+1.7$ & $-3.0$ & $-8.2$ & $+27.5$ & $-4.6$ & $-5.4$ \\
\bottomrule
\end{tabular*}
\begin{tablenotes}[flushleft]
\footnotesize
\item Each entry is $\Delta$ = (that fusion method's success rate under that layer configuration)
$-$ (the single-layer success rate of the \emph{shallowest} layer in the fused set), in percentage
points, both averaged over $3$ training seeds. Positive means fusion beat its shallowest
constituent.
\end{tablenotes}
\end{threeparttable}
\end{table*}

\begin{figure*}[t]
  \centering
  \includegraphics[width=\linewidth]{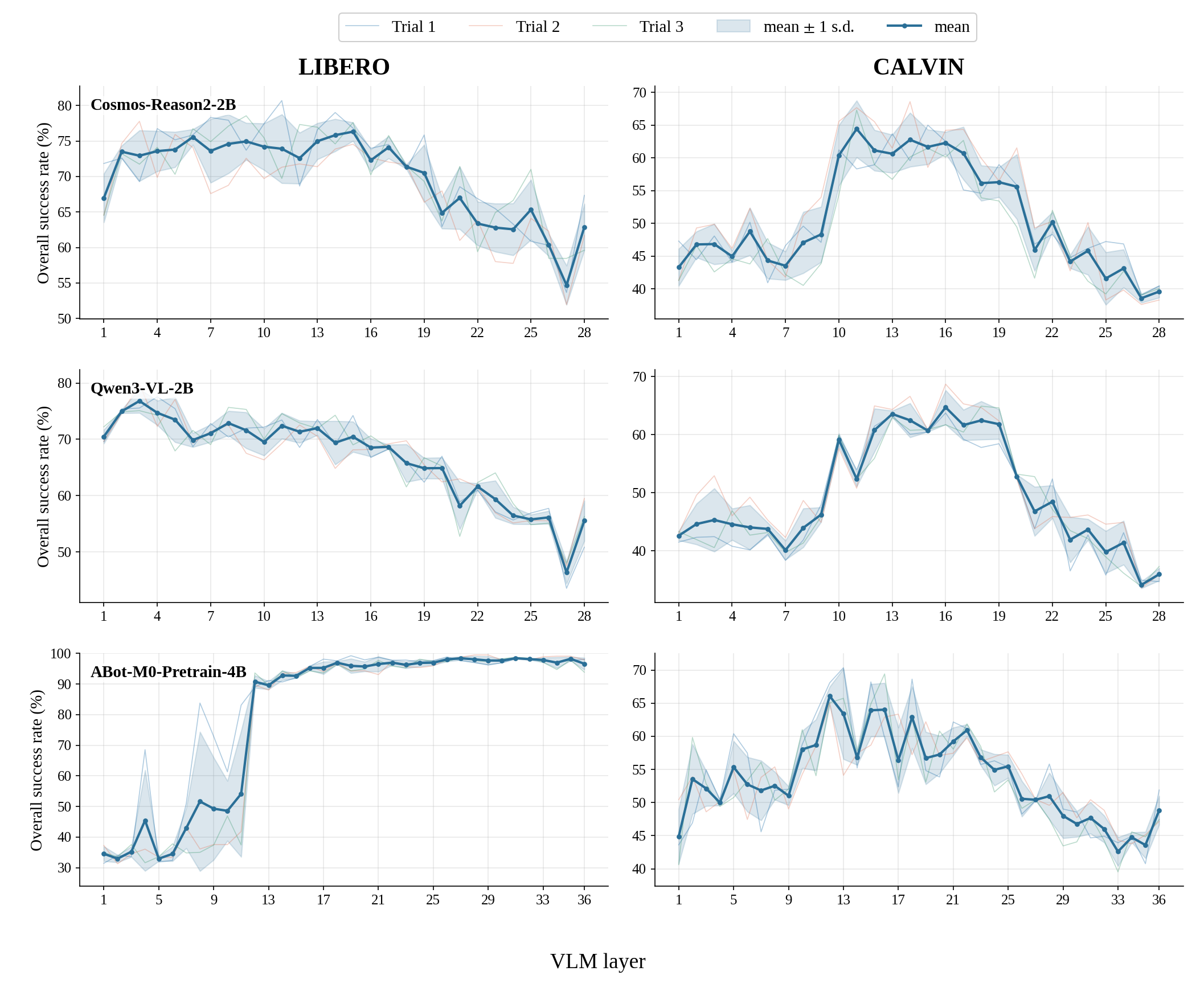}
  \caption{\textbf{Layer quality depends non-monotonically on depth and varies across training seeds.} Overall success rate versus VLM layer for three backbones (rows) on LIBERO (left) and CALVIN (right). Bold curves show the mean across three training seeds, shaded regions show one standard deviation, and faint curves show individual seeds. See Sec.~\ref{sec:layer_choice_oracle} for experiment details.}
  \label{fig:success_vs_layer}
\end{figure*}

\subsection{Exception Analysis}
\label{app:exception_analysis}
This section explains our experiment protocol for analyzing fusion-win exceptions in Table ~\ref{tab:fusion_vs_best}. 

We based our experiment on Abot-Pretrain base model and CALVIN benchmark because 5 out of 7 exceptions happen with CALVIN benchmark and 5 out of 7 happen with Abot-Pretrain base model.

To isolate the factors of architecture change and fusion inputs(Sec. ~\ref{sec:paradigm:exception}), we define the following ablated implementation of APT. APT-Repeat[$i$] has the exact action head architecture as standard APT but the action head's input space is just a tuple aggregating $m$ copies of a single-layer's outputs. $m$ is the number of layers allowed to be fused with the APT action head. Intuitively, each APT-Repeat[$i$] ablates the layer-fusion factor while keeping the architecture-change factor. We measure the performance of APT-Repeat[$i$] for all potential layers and compare it with the single-layer GR00T. The main result is reported in Fig.~\ref{fig:apt_repeat_vs_groot}.
We present two observations: (A) For each layer $i$, APT-REPEAT[$i$] consistently presents a higher CALVIN success rate than the single-layer GR00T action head trained with the corresponding layer. (B) The highest achievable performance with APT-REPEAT[$i$] is higher than APT-Last, APT-Best, and APT-even. Both observations imply that the performance gain of APT in Table ~\ref{tab:fusion_vs_best} can be obtained with architecture change exclusively without fusing any layer.

\begin{figure}[t]
  \centering
  \includegraphics[width=\linewidth]{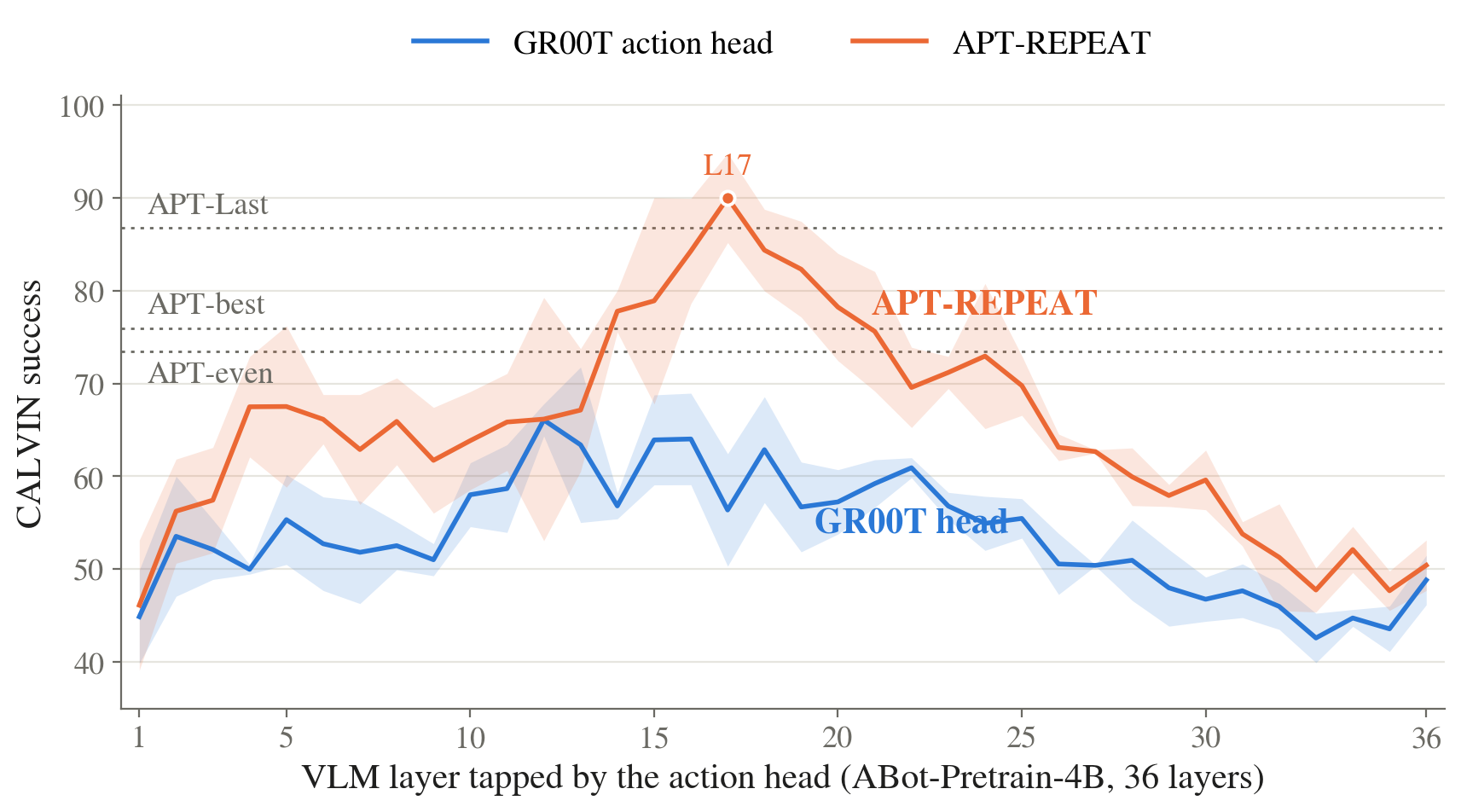}
  \caption{\textbf{APT-Repeat[$i$] versus the single-layer GR00T policy across all 36 layers of Abot-Pretrain on CALVIN.} Curves show the mean over the three policy seeds and shaded bands show one standard deviation. Success is reported in percent. Dotted lines mark the three APT fusion configurations (APT-Last, APT-best, APT-even) under the same backbone and benchmark. GR00T values are taken from the single-layer oracle sweep of Sec.~\ref{sec:layer_choice_oracle}; APT-Repeat values are measured with the protocol os Appendix ~\ref{app:exception_analysis}.}
  \label{fig:apt_repeat_vs_groot}
\end{figure}

\subsection{Tightness Analysis of Bound ~\ref{eq:infoNCE_old_view}}
\label{app:bound_analysis}
This section derives the finite-$N$ slack of bound~\eqref{eq:infoNCE_old_view} and reports its empirical estimate under the interpretation framework of ~\cite{oord2019representationlearningcontrastivepredictive}. Let $r(z,a)=\frac{p(z\mid a)}{p(z)}$ be the density ratio between $P_{Z\mid A}$ and $P_Z$. \cite{oord2019representationlearningcontrastivepredictive} show that the optimal critic satisfies $\exp h^\star(z,a)\propto r(z,a)$, with a constant that may depend on $a$, and that
\begin{equation}
\label{eq:slack_opt_loss}
\mathcal{L}_{\text{InfoNCE}}(\phi,h^\star)\approx\mathbb{E}_{P_{Z,A}}\left[\log\left(1+\frac{N-1}{r(Z,A)}\right)\right].
\end{equation}
Subtracting Eq.~\eqref{eq:slack_opt_loss} from $\log N$ and writing $r=r(Z,A)$,
\begin{equation}
\label{eq:slack_steps}
\begin{aligned}
\log N-\mathcal{L}_{\text{InfoNCE}}(\phi,h^\star)
&\approx \mathbb{E}_{P_{Z,A}}\left[\log N-\log\frac{r+N-1}{r}\right]\\
&= \mathbb{E}_{P_{Z,A}}\left[\log r-\log\frac{r+N-1}{N}\right]\\
&= \mathbb{E}_{P_{Z,A}}\left[\log r\right]\\
&\quad-\mathbb{E}_{P_{Z,A}}\left[\log\left(1+\frac{r-1}{N}\right)\right].
\end{aligned}
\end{equation}
The first line rewrites $1+\frac{N-1}{r}$ as $\frac{r+N-1}{r}$. The second line splits $\log N-\log\frac{r+N-1}{r}=\log r-\log\frac{r+N-1}{N}$, and the third uses $\frac{r+N-1}{N}=1+\frac{r-1}{N}$. The first expectation in the last line is the mutual information, since
\begin{equation}
\label{eq:slack_mi}
\mathbb{E}_{P_{Z,A}}\left[\log r(Z,A)\right]
=\mathbb{E}_{P_{Z,A}}\left[\log\frac{p(Z\mid A)}{p(Z)}\right]
=I(Z, A).
\end{equation}
Substituting Eq.~\eqref{eq:slack_mi} into Eq.~\eqref{eq:slack_steps} and rearranging gives
\begin{equation}
\label{eq:slack_def}
\begin{aligned}
&I(Z, A)-\big(\log N-\mathcal{L}_{\text{InfoNCE}}(\phi,h^\star)\big)\approx\Delta_N,\\
&\Delta_N=\mathbb{E}_{P_{Z,A}}\left[\log\left(1+\frac{r(Z,A)-1}{N}\right)\right].
\end{aligned}
\end{equation}
So $\Delta_N$ is the amount by which the InfoNCE score of the optimal critic falls below $I(Z, A)$. Its magnitude quantifies the tightness of the bound ~\ref{eq:infoNCE_old_view} under the interpretation framework of ~\cite{oord2019representationlearningcontrastivepredictive}.

\paragraph{Empirical expression.} We replace $r$ by the ratio induced by the fitted critic $h$. Because the proportionality constant depends on $a$, it is removed separately for every query. Let $\{(z_i,a_i)\}_{i=1}^{n}$ be held-out pairs not used to train $h$. For each query $a_i$ we draw one set of latents from $P_Z$, excluding $z_i$: a normalization set $\mathcal{N}^{\mathrm{norm}}_i$ of size $m_1$. The normalized ratios are
\begin{equation}
\label{eq:slack_critic_ratio}
\widehat{r}_{ij}=\frac{\exp h(z_j,a_i)}{\frac{1}{m_1}\sum_{k\in\mathcal{N}^{\mathrm{norm}}_i}\exp h(z_k,a_i)}.
\end{equation}
Plugging $\widehat{r}_{ii}$ into the slack definition in Eq.~\eqref{eq:slack_def} gives
\begin{equation}
\label{eq:slack_estimates}
\widehat{\Delta}_N=\frac{1}{n}\sum_{i=1}^{n}\log\left(1+\frac{\widehat{r}_{ii}-1}{N}\right)
\end{equation}

Fig.~\ref{fig:infonce_slack_abot_libero} reports $\widehat{\Delta}_N$ for LIBERO and Abot-Pretrain. Two empirical observations hold: (A) the slack sizes are empirically small comparing to absolute magnitudes of infoNCE score. More quantitatively, the ratio of \emph{slack / infoNCE score} falls within the range of $[0.067, 0.099]$ (B) The empirically estimated slacks do not change the overall shape of the \emph{infoNCE score v.s. layer} curve observed in fig. ~\ref{fig:success_vs_infoNCE_grid} -- its overall non-monotone-decreasing still holds even empirically estimated slacks are taken into accounts. 

\begin{figure}[t]
  \centering
  \includegraphics[width=0.85\linewidth]{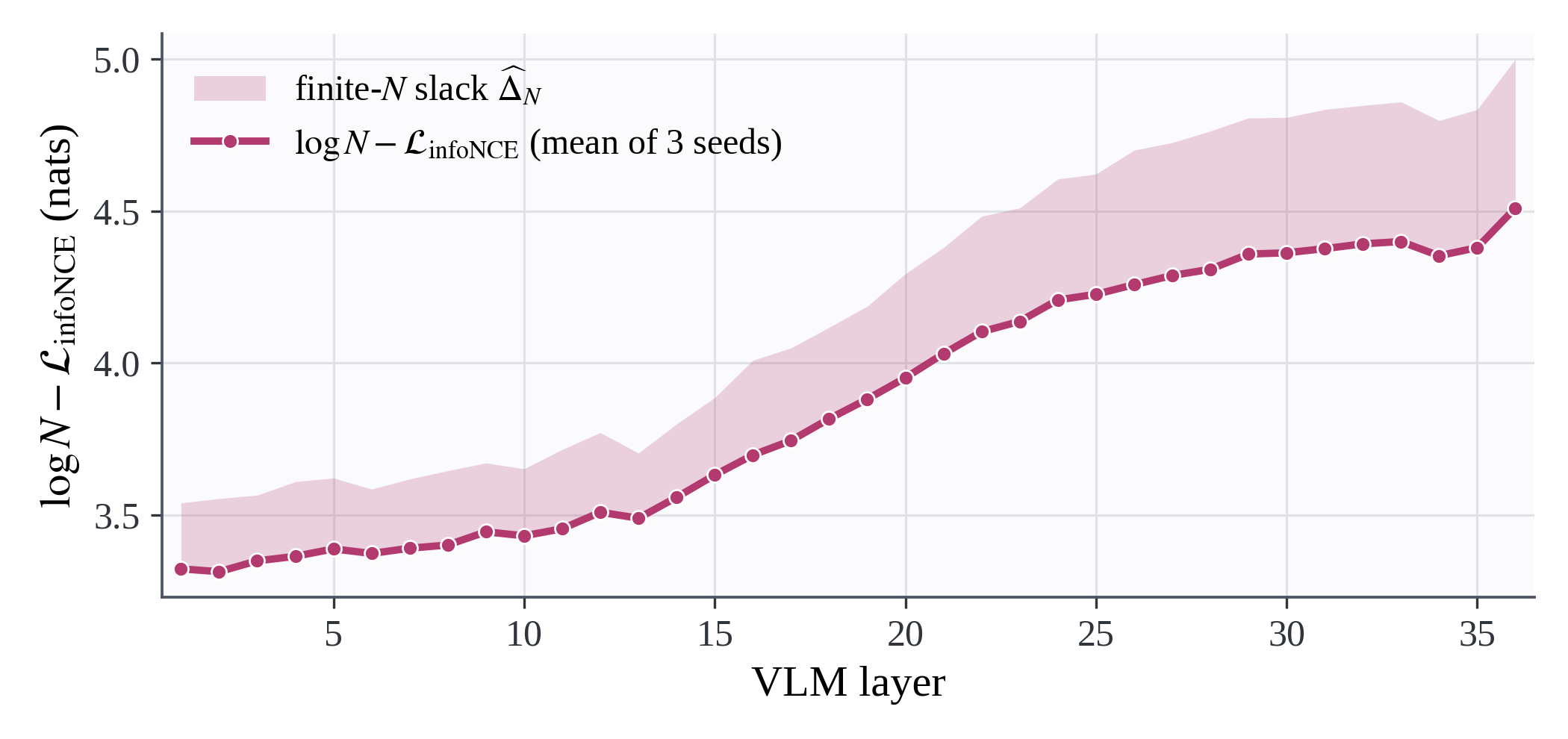}
  \caption{\textbf{Finite-$N$ slack of the InfoNCE score for Abot-Pretrain on LIBERO.} The line shows $\log N - \mathcal{L}_{\text{infoNCE}}$ averaged over three proxy seeds, and the shaded band extends it upward by the direct slack estimate $\widehat{\Delta}_N$ of Eq.~\eqref{eq:slack_estimates}, also averaged over the seeds. Smaller $\Delta_N$ means empirically tighter bound in Eq. ~\ref{eq:infoNCE_old_view}. Each critic uses $N=512$ and is scored on $n=16{,}384$ held-out pairs, with a normalization set of $m_1=15{,}383$ of the other $16{,}383$ held-out latents per query.}
  \label{fig:infonce_slack_abot_libero}
\end{figure}

\subsection{Statistical Test Measure of Fusion Advantage}
\label{app:stats_fusion_analysis}

This subsection explains how we test whether each fusion configuration outperforms the optimal single layer. We analyze each benchmark--backbone setting separately. Since outperforming the optimal layer requires outperforming every layer, we first compare fusion with each single layer and then combine the comparisons.

For a fixed fusion configuration $F$ and single layer $l$, let $\bar{S}_F$ and $\bar{S}_l$ denote their mean success rates across training runs. Let $s_F$ and $s_l$ denote their sample standard deviations, and let $n_F$ and $n_l$ denote their numbers of runs. Our experiments use $n_F=n_l=3$. We adapt the approach of Welch's two-sample $t$-test \cite{welch1947generalization}.

The test statistic is

$$
T_l =
\frac{\bar{S}_F-\bar{S}_l}
{\sqrt{s_F^2/n_F+s_l^2/n_l}}.
$$

This statistic measures the observed advantage of fusion relative to the uncertainty in that difference. A larger positive value provides stronger evidence that fusion outperforms layer $l$.

The cumulative distribution function can be written explicitly as

$$
F_{t_{\nu_l}}(T_l)
=
\frac{\Gamma\!\left((\nu_l+1)/2\right)}
{\sqrt{\nu_l\pi}\,\Gamma\!\left(\nu_l/2\right)}
\int_{-\infty}^{T_l}
\left(1+\frac{u^2}{\nu_l}\right)^{-(\nu_l+1)/2}\,du,
$$

where $\Gamma$ is the gamma function. Therefore, the one-sided p-value is

$$
\boxed{
p_l
=
\frac{\Gamma\!\left((\nu_l+1)/2\right)}
{\sqrt{\nu_l\pi}\,\Gamma\!\left(\nu_l/2\right)}
\int_{T_l}^{\infty}
\left(1+\frac{u^2}{\nu_l}\right)^{-(\nu_l+1)/2}\,du.
}
$$

This integral measures the probability of a test statistic at least as large as $T_l$ under equal expected performance. For fixed $\nu_l$, increasing $T_l$ reduces the remaining area under the curve, giving a smaller p-value and stronger evidence for fusion's advantage.

Under the Welch approximation, $p_l$ is the probability of obtaining a test statistic at least as large as the observed value if fusion and layer $l$ have equal expected performance. Thus, a small p-value indicates that the observed advantage would be unusual under equal performance.

We combine the comparisons using an intersection--union test:

$$
p_{\mathrm{overall}} = \max_l p_l.
$$

The maximum represents the least convincing comparison between fusion and a single layer. We conclude that fusion has a statistically significant advantage over the optimal layer when $p_{\mathrm{overall}}<0.05$. This requires every layer comparison to satisfy $p_l<0.05$. No additional correction across layers is required for this combined test.

Lower overall p-values indicate stronger statistical evidence for fusion's advantage, but do not measure the size of that advantage. A value above $0.05$ means that the available runs do not establish superiority; it does not establish equivalence or inferiority. With three runs per configuration, this analysis is approximate. We report nominal p-values for each fusion configuration, without adjustment across different fusion configurations.

\begin{table}[t]
\centering
\caption{\textbf{All fusion methods other than APT do not demonstrate statistically significant advantage over the single optimal layer.} Each cell reports the p-value for a fusion configuration compared with the optimal single layer. Lower values indicate stronger statistical evidence that fusion outperforms the optimal single layer; \(p<0.05\) indicates a statistically nominal advantage. See Appendix ~\ref{app:stats_fusion_analysis} for calculation of p-value.
}
\label{tab:fusion_pvalues}
\small
\begin{tabular}{@{}lllccc@{}}
\toprule
Benchmark & Backbone & Method & Best & Even & Last \\
\midrule
LIBERO & Cosmos & APT & .666 & .969 & .790 \\
 &  & VLA-Adapter & .877 & .985 & 1.000 \\
 &  & BEHAVIOR & .983 & .923 & 1.000 \\
\midrule
LIBERO & Qwen & APT & .900 & .990 & .870 \\
 &  & VLA-Adapter & .942 & .998 & .999 \\
 &  & BEHAVIOR & .902 & .898 & .999 \\
\midrule
LIBERO & ABot & APT & .934 & .996 & .941 \\
 &  & VLA-Adapter & .921 & 1.000 & .299 \\
 &  & BEHAVIOR & .564 & .907 & .457 \\
\midrule
CALVIN & Cosmos & APT & .945 & .999 & .223 \\
 &  & VLA-Adapter & .844 & .977 & .999 \\
 &  & BEHAVIOR & .863 & .882 & .999 \\
\midrule
CALVIN & Qwen & APT & .935 & .997 & .437 \\
 &  & VLA-Adapter & .894 & .993 & 1.000 \\
 &  & BEHAVIOR & .984 & .987 & 1.000 \\
\midrule
CALVIN & ABot & APT & .057 & .081 & .012 \\
 &  & VLA-Adapter & .983 & .989 & .999 \\
 &  & BEHAVIOR & .924 & .997 & .997 \\
\bottomrule
\end{tabular}
\end{table}

\section{Experimental Setup and Reproducibility}
\label{app:experimental_setup}

This section reports the settings used by the experiments in the paper.  Unless
stated otherwise, a setting is shared by all backbones and both benchmarks.

\subsection{Benchmarks and dataset splits}
\label{app:benchmarks_splits}

\paragraph{LIBERO.}
We use all ten tasks from each of LIBERO-Spatial, LIBERO-Object, LIBERO-Goal,
and LIBERO-10, for 40 tasks in total. See ~\cite{liu2023libero} for the definition of each task.

\paragraph{CALVIN.}
Training starts from \path{InternRobotics/InternData-Calvin_ABC}, which contains
15,976 trajectories and 961,081 frames.  It contains all 34
canonical training tasks ~\cite{mees2022calvinbenchmarklanguageconditionedpolicy}.  We use a frame-balanced,
task-stratified trajectory-samplings strategy.  It contains 4,608
whole trajectories and 274,466 frames.  The target was the LIBERO total of
273,465 frames; the small excess comes from retaining whole trajectories.  No
frame is dropped within a selected trajectory.  The training mixture keeps all
34 tasks.

Closed-loop CALVIN evaluation uses the 27 tasks that can be reset as standalone
episodes by the repository's state constructor.  It excludes
\path{lift_blue_block_drawer}, \path{lift_pink_block_drawer},
\path{lift_red_block_drawer}, \path{place_in_drawer},
\path{place_in_slider}, \path{stack_block}, and \path{unstack_block}.  Thus the
evaluated set contains open/close drawer; left/right slider motion; left/right
rotation and pushing for all three blocks; lifting each block from the table or
slider; on/off for the LED and lightbulb; and push into drawer.

\begin{table*}[t]
\centering
\caption{Training data and fixed trajectory-level validation splits.  ``Samples'' are valid
one-observation, 16-action windows.}
\label{tab:data_split_counts}
\small
\setlength{\tabcolsep}{4pt}
\begin{tabular}{@{}lrrrrrr@{}}
\toprule
Dataset & Demos & Frames & Samples & Train demos & Val. demos & Train/val. samples \\
\midrule
LIBERO-Spatial & 432 & 52,970 & 46,490 & 410 & 22 & 44,116 / 2,374 \\
LIBERO-Object  & 454 & 66,984 & 60,174 & 431 & 23 & 57,192 / 2,982 \\
LIBERO-Goal    & 428 & 52,042 & 45,622 & 407 & 21 & 43,497 / 2,125 \\
LIBERO-10      & 379 & 101,469 & 95,784 & 360 & 19 & 90,984 / 4,800 \\
\midrule
LIBERO total   & 1,693 & 273,465 & 248,070 & 1,608 & 85 & 235,789 / 12,281 \\
CALVIN ABC subset & 4,608 & 274,466 & 205,346 & 4,378 & 230 & 195,438 / 9,908 \\
\bottomrule
\end{tabular}
\end{table*}

For policy training, each dataset is split independently at the trajectory
level.  Five percent of trajectories are held out with fixed split seed
20260708.  This seed is separate from the model-training seed, so every layer,
method, and training seed uses the same train/validation membership.  The
validation data are used for validation loss and for proxy-feature extraction;
they are not used to update a policy.  The final checkpoint is not selected by
validation loss.

On both benchmarks, a task success rate is the number of successful completed
episodes divided by the number of completed episodes.  LIBERO uses the
benchmark environment's success predicate.  CALVIN compares reset and current
environment information with its task oracle.  An episode stops on the first
success.  A suite score is the unweighted mean over its tasks.  The reported
LIBERO score is the unweighted mean over all 40 task rates, and the CALVIN
score is the unweighted mean over the 27 standalone task rates.  It is not a
frame-weighted or demonstration-weighted quantity.
\begin{table*}[t]
\centering
\caption{Actual backbone--benchmark settings.  Width is the language-tower hidden size.
The seed order is the order used in per-seed tables below.}
\label{tab:pair_settings}
\small
\setlength{\tabcolsep}{4pt}
\begin{tabular}{@{}llcrll@{}}
\toprule
Benchmark & Backbone checkpoint & Layers & Width & Oracle seeds & Episodes/task \\
\midrule
LIBERO & \path{nvidia/Cosmos-Reason2-2B} & 28 & 2048 & 42, 1234, 2718 & 10 \\
LIBERO & \path{Qwen/Qwen3-VL-2B-Instruct} & 28 & 2048 & 42, 1234, 2718 & 10 \\
LIBERO & local ABot-Pretrain VLM & 36 & 2560 & 42, 1234, 4321 & 10 \\
CALVIN & \path{nvidia/Cosmos-Reason2-2B} & 28 & 2048 & 42, 1234, 2718 & 15 \\
CALVIN & \path{Qwen/Qwen3-VL-2B-Instruct} & 28 & 2048 & 42, 1234, 2718 & 15 \\
CALVIN & local ABot-Pretrain VLM & 36 & 2560 & 42, 1234, 2718 & 15 \\
\bottomrule
\end{tabular}
\end{table*}

\subsection{Inputs and latent-interface construction}
\label{app:inputs_interface}

Each input has one current external-camera image and one current wrist-camera
image.  LIBERO stores both at $256\times256$ pixels.  CALVIN stores the base
view at $200\times200$ and the wrist view at $84\times84$.  A replayed transform
applies the same crop parameters to both views: resize the shortest edge to
256, take a random 0.95-fraction crop during training or a centered crop during
evaluation, and resize the shortest edge to 256 again.  Training also uses
color jitter with brightness 0.3, contrast 0.4, saturation 0.5, and hue 0.08;
evaluation does not.  The Qwen processor then rescales RGB values by $1/255$
and normalizes each channel with mean and standard deviation 0.5.

The prompt contains the two images followed by the dataset task description in
the Qwen chat format.  Language is lower-cased and punctuation is removed.  No
generation prompt is appended.  The current proprioceptive state is also sent
directly to the action head.  LIBERO uses eight scalars
$(x,y,z,\mathrm{roll},\mathrm{pitch},\mathrm{yaw},g_1,g_2)$ because its source
state contains two gripper entries.  CALVIN uses seven scalars with one gripper
entry and omits joint angles.  State is percentile-normalized, padded to 132
dimensions, and dropped as a whole with probability 0.2 during training.

For a single-layer interface, we take the complete mixed text--image hidden
sequence after language layer $l$.  Its shape is $B\times S\times d$, where
$d=2048$ for Cosmos and Qwen and $d=2560$ for ABot.  $S$ is the dynamic token
length produced by the tokenizer and image grid.  The collator pads it to the
largest $S$ in that minibatch, and the attention mask marks real tokens.  We keep text tokens, image tokens, and chat
control tokens.  We do not select only image tokens and do not mean-pool them.
The raw hidden sequence is produced before the action-head normalization.  The
standard action head then applies a trainable LayerNorm over the last dimension;
the optional VLM self-attention stage is an identity in these runs.  There is
no separate latent projection before cross-attention.  Proxy extraction reads
the same raw, pre-action-head-LayerNorm sequence.  Its critic or metric performs
the pooling and normalization described below.

Layer numbers are one-based.  Layer 1 is the output of the first language
decoder block, not the token embedding.

\subsection{Action representation and prediction head}
\label{app:action_head}

Both benchmarks use seven action scalars per step: end-effector translation
$(\Delta x,\Delta y,\Delta z)$, end-effector rotation
$(\Delta\mathrm{roll},\Delta\mathrm{pitch},\Delta\mathrm{yaw})$, and the
gripper command.  These are the native simulator/controller coordinates stored
by each converted dataset; the repository records no further world-to-base
coordinate transform.  LIBERO sends the delta pose to the robosuite OSC
controller.  CALVIN reads the source \path{action.delta_ee_pos},
\path{action.delta_ee_rot}, and \path{action.gripper} fields.  Each scalar is
normalized to $[-1,1]$ with the training dataset's 1st and 99th percentiles.
Outliers are clipped.  The model predicts 16 valid action steps.  The processor
pads them to an internal $40\times132$ tensor and masks all padded elements.
At rollout time the first eight predicted steps are executed open loop, then
the policy observes again and replans.

The standard head is GR00T N1.7's 16-block AlternateVLDiT.  It has 32 attention
heads of width 48, for an inner width of 1536, and alternates cross-attention
and self-attention.  Eight blocks cross-attend to the selected VLM sequence.
An embodiment-conditioned MLP embeds the padded state as one state token; a
separate embodiment-conditioned encoder embeds the noisy action tokens.  Thus
the policy is conditioned on the selected VLM sequence, current proprioception,
the noisy action trajectory, diffusion time, and embodiment ID.  It does not
read raw images outside the frozen VLM.

Let $a$ be the normalized padded action, $\epsilon\sim\mathcal N(0,I)$, and
$u\sim\mathrm{Beta}(1.5,1.0)$.  Training samples
$t=(1-u)\,0.999$, forms $a_t=(1-t)\epsilon+ta$, and predicts the velocity
$a-\epsilon$.  The loss is the mean squared error over only the valid 16 by 7
entries.  Continuous time is bucketed into 1,000 bins for the timestep
embedding.  Inference starts from standard Gaussian noise and uses four Euler
updates with step $1/4$ at times $0,1/4,1/2,3/4$.  The decoded actions are
un-padded and un-normalized before execution.

\subsection{Policy-training configuration}
\label{app:policy_training}

Every policy uses AdamW (PyTorch implementation), learning rate $10^{-4}$, a
cosine schedule, 400 warmup steps (5\% of training), weight decay $10^{-5}$,
and gradient-norm clipping at 1.0.  Training lasts 8,000 optimizer steps with
global batch size 32, one H100 GPU, no gradient accumulation, and four data
loader workers.  Training and validation use bfloat16; TF32 is enabled and
FP16 is disabled.  Validation batch size is 16.  The main comparison always
uses the step-8,000 checkpoint.  Validation loss is logged but never used to
pick a checkpoint.  Operational save and validation intervals differed across
old drivers: validation was normally every 2,000 steps for the early LIBERO
Cosmos/Qwen sweeps and every 4,000 steps for ABot and CALVIN; some fusion
drivers evaluated validation every 1,000 steps.  This difference cannot affect
the stated checkpoint rule.

The three common policy seeds are 42, 1234, and 2718.  The only exception is
the LIBERO ABot oracle sweep: its seed-2718 sweep was incomplete, so the three
complete oracle curves use 42, 1234, and 4321.  Fusion experiments use 42,
1234, and 2718 for every pair.  Proxy critics and metric sampling also use 42,
1234, and 2718.  The optimizer, schedule, batch size, step count, action head,
image pipeline, validation fraction, and checkpoint rule are shared.  The VLM
depth and width vary by backbone; raw image size, state dimension, dataset,
rollout limit, and episode count vary by benchmark; only the layer interface
and fusion module vary by method.

\subsection{Single-layer oracle sweep}
\label{app:oracle_sweep}

We evaluate every language layer: 1--28 for Cosmos and Qwen and 1--36 for
ABot.  Each layer and seed is an independent run with the same frozen backbone
checkpoint and a freshly initialized action head.  It is trained for the full
8,000-step budget and evaluated closed loop.  The oracle-quality curve is the
arithmetic mean of overall benchmark success over the three policy seeds at
each layer.  The optimal layer is the exact maximum of this seed-mean curve.
The implementation visits layers in ascending order, so an exact numerical tie
is resolved in favor of the shallower layer.

The same closed-loop evaluations are used both to choose the maximum and to
report its performance.  There is no second held-out set for choosing the
oracle layer.  The observed maximum is therefore an in-sample empirical oracle,
not an unbiased estimate of performance after model selection.

\subsection{Fusion implementations and layer configurations}
\label{app:fusion_impl}

All layer lists below are one-based and sorted from shallow to deep before they
are assigned to the action head.  APT uses 16 layers because it has one
injection per expert block.  VLA-Adapter and BEHAVIOR use eight layers because
AlternateVLDiT has eight cross-attention blocks.

\paragraph{Implementations.}
APT uses one self-attention stream over the concatenated VLM, state, and action
tokens.  Before expert block $i$, it adds
$\sigma(g_i)P(\mathrm{LN}(z_{l_i}))$ to the VLM-token slots.  $P$ is one shared
projection from the VLM width to 1536.  Each scalar gate starts at zero, so its
sigmoid starts at 0.5.  Relative to the original APT implementation, ours uses
full bidirectional attention rather than block-causal attention and the shared
GR00T flow-matching objective rather than two-stage diffusion.

Our VLA-Adapter implementation assigns one selected layer to each of the eight
cross-attention blocks.  The assignment is fixed; it has no learned fusion gate
or layer-mixing weights.  Cross-depth information can still propagate through
the action stream's residual state.  This fixed implementation is the code used
for the reported values.

BEHAVIOR gives each of the eight cross-attention blocks separate scalar mixtures
over the same eight selected layers for keys and values, plus per-head key and
value biases.  Each block starts with a one-hot weight on its corresponding
selected layer and zero bias.  It mixes LayerNormed hidden states and uses its
own key/value projections.  Relative to the original method, the VLM is frozen,
the mixture is restricted to eight selected layers instead of all depths, VLM
KV caches are not reused, and the objective and action head follow GR00T.

\begin{table*}[t]
\centering
\caption{Exact Best layer sets.  APT uses Best-16; VLA-Adapter and BEHAVIOR use Best-8.}
\label{tab:best_layer_sets}
\scriptsize
\setlength{\tabcolsep}{3pt}
\begin{tabular}{@{}lll@{}}
\toprule
Pair & Best-16 & Best-8 \\
\midrule
LIBERO/Cosmos & 1,2,4,5,6,7,8,9,10,11,13,14,15,16,17,19 & 5,6,8,11,12,14,15,17 \\
LIBERO/Qwen & 1,2,3,4,5,6,7,8,9,10,11,12,13,14,15,17 & 2,3,4,5,8,9,11,13 \\
LIBERO/ABot & 18,21,22,24,25,26,27,28,29,30,31,32,33,34,35,36 & 26,27,28,29,31,32,33,35 \\
CALVIN/Cosmos & 3,5,8,9,10,11,12,13,14,15,16,17,18,19,20,22 & 10,11,12,13,14,15,16,17 \\
CALVIN/Qwen & 2,3,9,10,11,12,13,14,15,16,17,18,19,20,21,22 & 12,13,14,15,16,17,18,19 \\
CALVIN/ABot & 5,10,11,12,13,14,15,16,17,18,19,20,21,22,23,25 & 11,12,13,15,16,18,21,22 \\
\bottomrule
\end{tabular}
\end{table*}

\subsection{InfoNCE critic training and scoring}
\label{app:infonce_details}

For each backbone--benchmark pair and proxy seed, we extract 4,096 samples from
the fixed policy-validation trajectories.  LIBERO uses exactly 1,024 samples
from each suite; CALVIN uses its single dataset.  Each sample contains the raw
pre-action-head-LayerNorm token sequence and the normalized first $16\times7$
action values flattened to 112 dimensions.  A seeded permutation makes a
3,277/819 (80/20) critic train/held-out split.  For Abot-Pretrain on LIBERO, we
instead extract 19,456 samples (4,864 per suite) from 254 of the 1,693
trajectories (15\%, a superset of the policy-validation trajectories) and split
them 3,072/16,384.  This critic is trained for 50 epochs and evaluated after
every epoch with a single pass over the held-out set; its 16,384 held-out pairs
are also used in Appendix~\ref{app:bound_analysis}.

The latent encoder applies LayerNorm at the VLM width, a linear projection to
256, and learned-query four-head attention over valid tokens, followed by
LayerNorm.  The action encoder is an MLP $112\rightarrow256\rightarrow256$ with
GELU.  A learned $256\times256$ bilinear matrix starts at identity plus Gaussian
noise of standard deviation 0.01.  For action $a_i$ and latent sequence $z_j$,
the logit is
\[
s_{ij}=h(a_i)^\top Wg(z_j).
\]
There is no temperature or $1/\sqrt d$ scaling.  The diagonal is positive and
the other 511 entries in a batch are in-batch negatives.  Individual windows
are shuffled without trajectory grouping, so two samples in a batch may come
from the same trajectory.

The critic uses batch size $N=512$, Adam with learning rate $10^{-3}$ and zero
weight decay, 300 epochs, and gradient-norm clipping at 5.  It is evaluated
every 25 epochs.  At each evaluation it averages ten seeded reshufflings of the
held-out set and keeps the checkpoint with the lowest held-out cross-entropy.
The reported score is $\log(512)-\mathcal L_{\mathrm{heldout}}$ in nats.  It is
not a training-set score.  The full layer curve is computed independently for
seeds 42, 1234, and 2718 and then averaged arithmetically.  InfoNCE-argmax picks
the maximum of that mean curve, with a shallow-layer tie break.

\subsection{Other proxy metrics}
\label{app:other_proxies}

DiME uses the same raw token sequence, flattened normalized $16\times7$ action,
4,096 samples, seeded 80/20 split, and batch size 512 as InfoNCE.  Its latent
encoder is the same 256-dimensional four-head attention pool.  Its action MLP
also has width 256.  Both outputs are unit-$\ell_2$ normalized.  Separate RBF
kernels use learned log bandwidths initialized at zero, or $\sigma=1$.  DiME is
the mean entropy of five independently permuted action Gram matrices minus the
entropy of the aligned Hadamard-product Gram.  The kernel entropy uses
$\alpha=1$.  We maximize held-out DiME with Adam at $10^{-3}$, clip gradients
at 5, evaluate ten reshufflings every 25 epochs, and retain the largest held-out
value.  The six reported pair sweeps use 200 epochs.  The LIBERO/Cosmos DiME
column has only seeds 42 and 1234 and used separate truncated-backbone
extraction for each layer.  The other DiME settings use three seeds and the
one-forward all-layer extractor.  This is an actual protocol difference in the
stored results.

Prompt- and dataset-level matrix entropy are training-free and use the same
4,096 held-out samples.  For dataset entropy, valid tokens are mean-pooled
within each sample to obtain a $4096\times d$ matrix.  Each dimension is
centered across samples and each sample row is unit-$\ell_2$ normalized.  For
prompt entropy, the valid $S_i\times d$ tokens of each sample are centered by
dimension and normalized by token row, entropy is computed for that prompt,
and the 4,096 values are averaged.  Both use the linear Gram matrix, clamp its
entries below zero to zero, symmetrize it, compute eigenvalues in float64,
clamp negative eigenvalues to zero, trace-normalize with $\epsilon=10^{-12}$,
and report von Neumann entropy $-\sum_jp_j\log p_j$ in nats.  The code also
saves $\exp(H)$, but the correlations in the paper use entropy.  For a fixed
pair and seed, all proxies draw from the same fixed trajectory holdout and use
the same nominal 4,096-sample budget; their learned objectives and, for the
Cosmos DiME caveat above, extraction path differ.  Per-sample values become a
layer score by arithmetic averaging, and seed-level layer scores are then
averaged.

\subsection{Closed-loop evaluation procedure}
\label{app:closed_loop}

LIBERO requests ten episodes per task, runs ten environments in parallel, and
allows 720 environment steps.  CALVIN requests 15 episodes per task, runs five
environments in parallel, and allows 360 steps.  The evaluator executes eight
actions per policy query.  LIBERO uses the simulator's task initial states.
CALVIN enumerates valid deterministic initial conditions for each task, shards
them across environment workers, and cycles only after exhausting them.

The evaluation command did not pass an explicit environment seed.  Unless the
unrecorded shell variable \path{GR00T_EVAL_SEED} was set, the historical worker
therefore used nondeterministic simulator randomness.  The repository does not
contain evidence that it was set, so we treat evaluation seeds as unrecorded.
Methods share the same task lists, requested episode counts, rollout limits,
and initial-state procedure, but not a provably identical stream of random
initial states.

Parallel workers finish a batch together.  The saved result can therefore
contain slightly more completed episodes than requested when several workers
finish at once; CALVIN result files commonly contain 408--416 rather than
exactly $27\times15=405$ completed episodes.  The score always uses the actual
saved Boolean outcomes.  A time-limit failure counts as failure.  Infrastructure
failures are retried twice after the first attempt, including a policy-server
restart.  If all three attempts fail, the task is marked missing, excluded from
the aggregate, and the evaluator exits nonzero; such a run is not accepted as
a completed paper result.

\subsection{Statistical analysis and selection regret}
\label{app:statistics}

Every reported policy mean is an arithmetic mean over the three policy-training
seeds listed in Table~\ref{tab:pair_settings}.  A reported ``$\pm$ SD'' for
fusion uses the sample standard deviation with denominator $n-1$.  The shaded
band in Fig.~\ref{fig:success_vs_layer} uses NumPy's population standard
deviation with denominator $n$.  No confidence interval is reported anywhere
in the paper.  Fusion differences subtract two independent three-seed means;
we do not pair the seeds and do not propagate an uncertainty interval for the
difference.

For a seed-averaged oracle curve $q(l)$ and a selected layer $\hat l$, selection
regret is
\[
R=q(l^*)-q(\hat l),\qquad l^*=\arg\max_l q(l),
\]
in success percentage points.  The APT comparator uses the same first term and
subtracts its three-seed fusion mean, so its value may be negative.  We do not
attach a confidence interval to regret.  Spearman correlations are computed
between the seed-averaged policy curve and seed-averaged proxy curve, not
between individual-seed curves.

Training randomness is the variation across policy seeds.  Evaluation
randomness is the finite set of unseeded closed-loop rollouts within each
training seed.  It was not replicated separately, so the two sources of
variation cannot be decomposed from these experiments.  Raw task-level Boolean
outcomes are kept in each result JSON.

Tables~\ref{tab:fusion_seed_libero} and~\ref{tab:fusion_seed_calvin} give the
per-seed fusion results behind the means and standard deviations in the paper.
Each cell is seed 42 / 1234 / 2718, in percent.

\begin{table}[!t]
\centering
\caption{Per-seed success at the layer selected by the three-seed oracle curve.  Values are
percent.  The triple follows the seed order in Table~\ref{tab:pair_settings}.}
\label{tab:oracle_per_seed}
\scriptsize
\setlength{\tabcolsep}{4pt}
\begin{tabular}{@{}llcrcc@{}}
\toprule
Benchmark & Backbone & $l^*$ & Per-seed success & Mean & Sample SD \\
\midrule
LIBERO & Cosmos & 15 & 76.8 / 74.5 / 77.6 & 76.3 & 1.6 \\
LIBERO & Qwen   & 3  & 75.6 / 79.8 / 75.1 & 76.8 & 2.6 \\
LIBERO & ABot   & 31 & 98.2 / 98.3 / 98.8 & 98.4 & 0.3 \\
CALVIN & Cosmos & 11 & 58.3 / 67.7 / 67.3 & 64.4 & 5.3 \\
CALVIN & Qwen   & 16 & 63.7 / 68.7 / 61.7 & 64.7 & 3.6 \\
CALVIN & ABot   & 12 & 68.1 / 65.1 / 65.1 & 66.1 & 1.7 \\
\bottomrule
\end{tabular}

\vspace{2ex}

\centering
\caption{LIBERO fusion success by policy-training seed.}
\label{tab:fusion_seed_libero}
\scriptsize
\setlength{\tabcolsep}{4pt}
\begin{tabular}{@{}llccc@{}}
\toprule
Backbone & Method & Best & Even & Last \\
\midrule
Cosmos & APT & 75.0/67.9/80.5 & 72.5/74.8/72.1 & 74.5/65.7/77.8 \\
Cosmos & VLA-Adapter & 69.0/76.3/73.4 & 71.6/69.7/73.5 & 62.5/61.8/60.1 \\
Cosmos & BEHAVIOR & 71.7/70.0/73.7 & 72.6/58.3/69.4 & 54.1/56.4/57.8 \\
Qwen & APT & 76.8/69.4/71.9 & 69.4/67.1/71.6 & 58.6/75.7/72.0 \\
Qwen & VLA-Adapter & 73.4/74.7/70.1 & 62.7/66.5/61.1 & 56.9/58.4/51.7 \\
Qwen & BEHAVIOR & 76.9/70.4/67.0 & 75.1/67.5/74.3 & 54.5/45.5/50.3 \\
ABot & APT & 96.2/94.8/97.9 & 93.5/94.7/95.5 & 97.0/98.0/96.0 \\
ABot & VLA-Adapter & 97.3/98.2/96.5 & 44.3/45.0/47.0 & 98.5/98.6/98.9 \\
ABot & BEHAVIOR & 99.3/98.8/96.8 & 43.4/98.0/55.2 & 98.3/99.2/98.0 \\
\bottomrule
\end{tabular}

\vspace{2ex}

\centering
\caption{CALVIN fusion success by policy-training seed.}
\label{tab:fusion_seed_calvin}
\scriptsize
\setlength{\tabcolsep}{4pt}
\begin{tabular}{@{}llccc@{}}
\toprule
Backbone & Method & Best & Even & Last \\
\midrule
Cosmos & APT & 46.4/49.8/61.6 & 50.0/52.0/47.3 & 65.6/67.1/69.4 \\
Cosmos & VLA-Adapter & 56.8/65.5/56.1 & 53.1/59.2/54.6 & 42.3/44.6/36.9 \\
Cosmos & BEHAVIOR & 54.3/63.8/59.5 & 63.2/56.0/58.2 & 37.0/36.8/38.0 \\
Qwen & APT & 58.1/55.8/63.1 & 46.1/53.2/46.7 & 69.3/60.7/65.8 \\
Qwen & VLA-Adapter & 54.0/63.7/60.4 & 57.1/56.9/59.7 & 44.6/42.4/44.4 \\
Qwen & BEHAVIOR & 57.2/56.0/60.1 & 57.1/60.1/56.5 & 42.5/36.7/36.7 \\
ABot & APT & 73.9/78.1/75.9 & 70.0/77.4/72.9 & 91.0/82.7/86.5 \\
ABot & VLA-Adapter & 58.7/57.2/62.6 & 46.8/54.7/54.5 & 39.9/46.7/43.6 \\
ABot & BEHAVIOR & 51.1/61.8/61.9 & 48.3/51.8/54.8 & 39.3/40.1/48.2 \\
\bottomrule
\end{tabular}
\end{table}

\subsection{Computational cost and reproducibility resources}
\label{app:compute_repro}

All reported jobs use one NVIDIA H100 per run.  Multiple runs can execute at
the same time, but one GPU-hour always means one H100 occupied for one wall-clock
hour.  Oracle cost is the sum of the recorded \path{seconds} fields for every
successful layer's 8,000-step training job and closed-loop evaluation job over
three seeds.  Thus it includes policy training and evaluation.  InfoNCE cost is
the recorded end-to-end proxy sweep time and includes VLM feature extraction
and critic training.  Features are held on CPU between extraction and critic
training; there is no separately charged persistent feature-cache stage.
Table~\ref{tab:layer_sweep_cost} reports these full-sweep totals.  The aggregate
InfoNCE artifacts do not preserve per-phase timers, so extraction and critic
cost cannot be split after the fact.  The LIBERO/Cosmos cost cell uses the
three complete timing replicates with seeds 1234, 2718, and 31415; its reported
performance uses seeds 42, 1234, and 2718.  This is why summing the performance
runs' progress files does not reproduce 129.1 exactly.  The other oracle cost
cells use the performance runs listed above.

\end{document}